\documentclass[twoside]{article}

\usepackage[preprint]{aistats2027}

\usepackage[round]{natbib}

\usepackage[utf8]{inputenc}
\usepackage[T1]{fontenc}
\usepackage{graphicx}
\usepackage{dblfloatfix} % two-column float placement (table*, figure*)
\usepackage[breaklinks,hidelinks,hypertexnames=false]{hyperref}
\usepackage{url}            % simple URL typesetting
\usepackage{booktabs}       % professional-quality tables
\usepackage{amsfonts}       % blackboard math symbols
\usepackage{nicefrac}       % compact symbols for 1/2, etc.
\usepackage{xcolor}
\usepackage{placeins}
\usepackage{float}

\usepackage{needspace}
\makeatletter
\def\section{\@startsection{section}{1}{\z@}{-2.0ex plus -0.5ex minus -.2ex}{1.3ex \@plus 0pt \@minus .2ex}{\large\bfseries\raggedright\MakeUppercase}}
\makeatother

\usepackage{multirow}
\usepackage{enumitem}
\setlist[itemize]{nosep,leftmargin=*,topsep=2pt,partopsep=0pt}
\usepackage{amsmath}
\usepackage{amsthm}
\newtheorem{theorem}{Theorem}
\newtheorem{lemma}[theorem]{Lemma}
\newtheorem{proposition}[theorem]{Proposition}
\usepackage{amssymb} % math symbols (e.g., \gtrsim)
\usepackage{mathtools} % \coloneqq and related math tools
\newcommand{\R}{\mathbb{R}}

\newcommand{\struct}{\mathrm{struct}}

\newcommand{\softmax}{\mathrm{softmax}}
\newcommand{\relu}{\mathrm{ReLU}}

\newcommand{\forecast}{\mathcal{L}_{\mathrm{forecast}}}
\newcommand{\graphloss}{\mathcal{L}_{\mathrm{graph}}}

\newcommand{\totalloss}{\mathcal{L}_{\mathrm{total}}}

\newcommand{\tem}{\textsc{TEM}}
\newcommand{\Etem}{E_{\mathrm{TEM}}}
\newcommand{\alignloss}{\mathcal{L}_{\mathrm{align}}}

\newlength{\casefigH}

\begin{document}

\runningtitle{Structure-Aware Graph Abstention for Reliable Selective Forecasting}
\runningauthor{Xie and Alsinglawi}

\twocolumn[

\aistatstitle{Structure-Aware Graph Abstention for Reliable Selective Forecasting}

\aistatsauthor{ Jianxiang Xie \And Belal Alsinglawi }

\aistatsaddress{ University of New South Wales \And Zayed University } ]

\begin{abstract}
Selective forecasting abstains on high-risk test windows under a retained-coverage budget. Existing gates such as \tem{} \citep{brusokas2025time} score each forecast as a whole; for multivariate outputs, trajectories can look plausible while violating dependencies among variables. We treat \emph{instance-level plausibility} and \emph{relational consistency} as distinct reliability axes and operationalize the latter via a learned sparse graph and a Dirichlet-style structural energy $E_{\struct}(\hat{Y},A)$, trained with error-weighted graph regularization and score--error alignment (Section~\ref{sec:method}). On seven long-horizon benchmarks and four backbones, structural gating often reduces selective MSE versus \tem{} at matched coverage, with the largest gains where cross-variable structure appears more informative in our benchmarks; gains are not universal, indicating a complementary abstention signal. Table~\ref{tab:selective_mse_main} is a Protocol~A ranking diagnostic (seed~2024); three-seed deployable Protocol~B on an aligned subset is in Table~\ref{tab:multiseed_protocol_b} (full validation$\to$test grids: Appendix~\ref{sec:appendix_supp_ablation}).
\end{abstract}

\section{Introduction}
\emph{Selective forecasting} adds an abstention layer on top of point prediction: under a retained-coverage budget, the system withholds forecasts scored as high risk, trading how often predictions are released for lower error on those that are released. We study \emph{coverage-constrained score gating} (lower abstention score $\Rightarrow$ more likely accepted), using scalar scores as selective-risk proxies rather than calibrated probabilities.

Most selective rules score each forecast \emph{as a whole}. Time--energy models (\tem{}) \citep{brusokas2025time}---our primary comparator, not our contribution---train an energy-based model on decoder trajectories after a deep forecaster and abstain on high instance energy. This is insufficient for multivariate outputs: each channel can look plausible while the joint future violates dependencies seen in training \citep{wu2020connecting,chen2022balanced,benidis2022forecasting}. \textbf{Selective forecasting for structured outputs therefore faces two questions: is the forecast plausible, and is it relationally consistent across variables?} Instance-level plausibility and relational consistency are distinct notions of reliability; \tem{} targets the former, while we explicitly score the latter.

We call violation of learned cross-variable structure \emph{structural deviation} (high structural energy; Section~\ref{sec:method}). Coupled load and weather panels illustrate the failure mode: co-evolving channels can drift out of sync on the horizon while marginal traces stay benign. Relational gating is most informative when cross-channel coupling is stable (Electricity, Weather, Traffic) and less so when instance error dominates or coupling is weak (Section~\ref{sec:experiments}).

We present SASF (\emph{structure-aware selective forecasting}): a lightweight adaptive graph head yields adjacency $A$ and structural energy $E_{\struct}(\hat{Y},A)$; auxiliary losses make $E_{\struct}$ rank selective risk (Section~\ref{sec:method}). Prior graph work mostly improves $\hat{Y}$ itself; we use dependencies as a \emph{decision signal} for abstention (Section~\ref{sec:related_work}). We compare graph-only $E_{\struct}$ to instance-level $\Etem$ under matched coverage (Section~\ref{sec:experiments}; deployable validation$\to$test thresholds in Appendix~\ref{sec:appendix_supp_ablation}).

\paragraph{Contributions.}
\begin{itemize}
\item We identify a failure mode of instance-level selective forecasting: forecasts can be individually plausible yet structurally deviant across variables.
\item We operationalize \emph{relational consistency} through learned $A$ and $E_{\struct}$, with error-weighted graph learning and score--error alignment so structural deviation predicts selective risk.
\item Empirically, relational gating helps most where cross-variable dependence is more informative for abstention and on selected backbones, while instance-level scores remain competitive elsewhere---relational abstention complements rather than replaces \tem{} (Tables~\ref{tab:selective_mse_main}--\ref{tab:ablation2_main}; multiseed deployable Protocol~B in Table~\ref{tab:multiseed_protocol_b} and Appendix~\ref{sec:appendix_supp_ablation}).
\end{itemize}

\paragraph{Scope.}
We evaluate on public forecasting benchmarks, not deployed cyber-physical testbeds, and do not claim formal safety verification, conformal guarantees, or anomaly-detection certification.

\Needspace{4\baselineskip}
\section{Related Work}\label{sec:related_work}
Strong multivariate forecasters span decomposition Transformers and modern attention designs \citep{wu2021autoformer,zhou2022fedformer,nie2022time,wu2022timesnet,zhou2021informer,zeng2023transformers}, typically optimized for average accuracy \citep{benidis2022forecasting}.
Selective prediction studies coverage--risk tradeoffs via abstention \citep{geifman2017selective,cortes2016learning}; conformal and ensemble methods target uncertainty for predictive sets or calibrated sets \citep{vovk2005algorithmic,romano2019conformalized,angelopoulos2021gentle,gal2016dropout,lakshminarayanan2017simple,gneiting2007strictly}. We instead compare learned \emph{scalar} abstention scores under explicit selective protocols (Section~\ref{sec:experiments}) and adopt the instance-energy \tem{} pipeline \citep{brusokas2025time,lecun2006tutorial,hyvarinen2005estimation} as baseline.

Prior work thus falls into two families: (i)~\emph{instance-level} selective scores over whole forecasts, and (ii)~graph structure used to improve point forecasts \citep{wu2020connecting,liu2022multivariate,chen2022balanced,chen2023multi,cai2024forecastgrapher,kim2023hierarchical,sriramulu2023adaptive,zhang2024irregular,yu2017spatio,kipf2017semi}. We use learned dependencies to decide \textbf{whether a forecast should be released}---the graph is a relational abstention signal, not solely a predictive inductive bias. Section~\ref{sec:method} instantiates this under retained-coverage constraints.

\section{Problem Formulation}
Selective forecasting is a \emph{risk-aware decision} layer on top of point forecasting: after producing $\hat{Y}$, the system decides whether to \emph{release} a forecast or abstain, trading retained coverage (how often predictions are issued) against error on released predictions.
This objective is complementary to fitting $f_\theta$: we choose an acceptance rule $g$ that minimizes \emph{post-abstention} risk---average loss on retained samples---subject to a nominal retained-coverage (equivalently, abstention-rate) constraint.

Let $X \in \mathbb{R}^{B \times L \times D}$ denote a batch of historical multivariate sequences, where $B$ is the batch size, $L$ is the look-back window, and $D$ is the number of variables. The forecasting task is to predict future trajectories
\[
Y \in \mathbb{R}^{B \times H \times D},
\]
with prediction horizon $H$. A forecasting model $f_\theta$ produces
\[
\hat{Y} = f_\theta(X) \in \mathbb{R}^{B \times H \times D}.
\]

Given ground-truth $Y$, we define the per-sample squared error that underlies selective risk:
\begin{equation}
e_b = \frac{1}{HD}\sum_{h=1}^{H}\sum_{d=1}^{D} (\hat{Y}_{b,h,d} - Y_{b,h,d})^2.
\label{eq:mse}
\end{equation}

An acceptance function (selective decision rule)
\[
g(X_b, \hat{Y}_b) \in \{0,1\}
\]
determines whether a prediction is retained ($g=1$) or rejected ($g=0$). Let $\mathcal{S} = \{b : g(X_b,\hat{Y}_b)=1\}$ denote the set of accepted samples. Empirical retained coverage and selective risk are
\begin{equation}
\mathrm{cov} = \frac{|\mathcal{S}|}{B}, 
\qquad
\mathrm{risk} = \frac{1}{|\mathcal{S}|} \sum_{b \in \mathcal{S}} e_b \quad (\text{when } |\mathcal{S}|>0).
\label{eq:rc}
\end{equation}

Given a target coverage level $c \in (0,1]$, we seek a rule $g$ that minimizes selective risk under an approximate retained-coverage constraint:
\begin{equation}
\min_{g} \ \mathrm{risk}
\quad \text{s.t.} \quad
\mathrm{cov} \approx c.
\label{eq:objective}
\end{equation}
We instantiate $\mathrm{cov}\approx c$ via the ranking and transfer protocols in Section~\ref{sec:experiments} (target retained fraction $\rho$, with $c=\rho$). The central question is not how to improve $f_\theta$, but how to score multivariate reliability when instance-level plausibility can miss structural deviation. We compare (i)~\tem{} energy $\Etem$ (instance-level, prior pipeline) and (ii)~structural energy $E_{\struct}$ (relational, Section~\ref{sec:method}).

\section{Method}
\label{sec:method}
% =========================

\subsection{Overview}

SASF augments backbone $f_\theta$ with (1) our graph head ($A$, $E_{\struct}$) and (2) the prior \tem{} branch ($\Etem$, for comparison). Selective decisions are made from the abstention scores according to the evaluation protocols in Section~\ref{sec:experiments}. The graph head is trained jointly with $f_\theta$.

\begin{figure*}[!t]
  \centering
  \includegraphics[width=\textwidth,trim=0 0 6 0,clip]{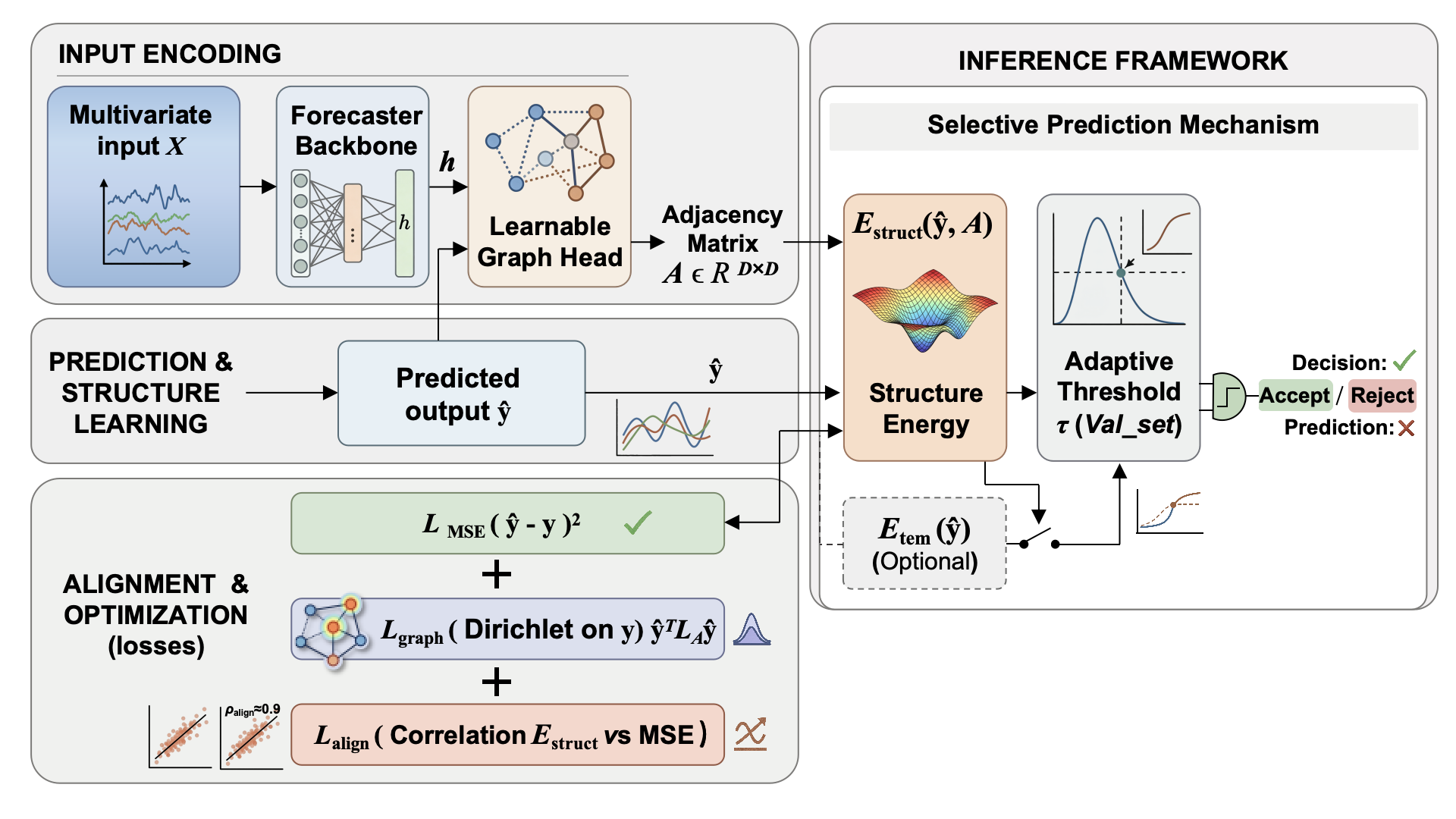}
  \caption{Overview: backbone $\hat{Y}$; baseline \tem{} energy $\Etem$ (prior pipeline); our graph head yields $A$ and $E_{\struct}(\hat{Y},A)$. Under matched test-split quantile gating, we obtain comparable selective evaluations for graph-only vs.\ \tem{}-only gates.}
  \label{fig:sasf-overview}
\end{figure*}

\subsection{Why relational energy?}

Instance-level gates can accept trajectories that look plausible in isolation yet break \emph{relational consistency}: $\Etem$ can be small while predicted co-evolution violates nominal dependencies in $A$. $E_{\struct}$ scores such \emph{structural deviation} for abstention, complementing $\Etem$. Section~\ref{sec:case_study} visualizes \emph{relational misses} ($\Etem$ accepts, $E_{\struct}$ rejects).

\subsection{Adaptive sparse graph head}

Learned $A$ encodes which variables co-evolve under the training distribution. The parameterization is intentionally lightweight: we use a simple \emph{global} graph not to maximize graph expressivity, but to isolate whether relational structure itself can serve as an abstention signal (rather than to advance graph-based forecasting). We use learnable node embeddings $U\in\R^{D\times d_g}$, affinities $M=\relu(UU^\top)$, and row-wise top-$k$ sparsification when $D{>}10$ ($k{=}10$), then row softmax
\begin{equation}
A = \softmax_{\mathrm{row}}(M_{\mathrm{sparse}}).
\end{equation}
For small $D$, we omit sparsification to avoid over-pruning. If no adaptive head is trained, one may instead build $A$ from training-set channel correlations (absolute Pearson), with the same top-$k$ and row normalization---used as a fallback in analysis scripts.

\subsection{Dirichlet-style structural energy}

If variables $i$ and $j$ are strongly connected in $A$, their predicted trajectories should co-evolve; we penalize large horizon disagreement on weighted edges (Dirichlet-style graph energy; Appendix~\ref{sec:appendix_theory_energy}). Given forecast batch $\hat{Y}\in\R^{B\times H\times D}$, fix $b\in\{1,\ldots,B\}$ and form pairwise squared gaps
\begin{equation}
\mathrm{dist}_{b,i,j}
=
\sum_{h=1}^{H}
\left(
\hat{Y}_{b,h,i} - \hat{Y}_{b,h,j}
\right)^2,
\end{equation}
and define the \emph{per-example} structural energy (a scalar abstention score for sample $b$) as this graph smoothness functional,
\begin{equation}
E_{\struct}(\hat{Y}_b,A)
=
\frac{1}{D^2}\sum_{i=1}^{D}\sum_{j=1}^{D}
A_{ij}\,\mathrm{dist}_{b,i,j}.
\label{eq:estruct}
\end{equation}
Large $E_{\struct}(\hat{Y}_b,A)$ indicates \emph{structural deviation} from relational consistency under $A$. At test time we sort $\{E_{\struct}(\hat{Y}_b,A)\}_{b=1}^B$ (written $E_{\struct}(\hat{Y},A)$ when batch-index dependence is clear). $A$ is a single global unsigned graph, so the score captures smoothness-based deviation rather than lags or sample-specific structure.

\paragraph{Theoretical interpretation.}
Although $A$ may be asymmetric after row-wise sparsification and normalization, the pairwise energy in Eq.~\eqref{eq:estruct} depends only on the symmetrized weights $W=(A+A^\top)/2$ (Lemma~\ref{lem:symmetrization}, Appendix~\ref{sec:appendix_theory_sym}). If $E_{\struct}(Y_b,A)\le\epsilon$ and $\lambda_{\max}(L_W)>0$, Theorem~\ref{thm:error_certificate} implies a one-sided certificate on the same per-sample selective error $e_b$ (Eq.~\eqref{eq:mse}):
\begin{equation}
\begin{aligned}
e_b &\ge \frac{D}{2H\,\lambda_{\max}(L_W)}
\Big[\sqrt{E_{\struct}(\hat{Y}_b,A)}-\sqrt{\epsilon}\Big]_+^{2}, \\
&\qquad [x]_+=\max(x,0).
\end{aligned}
\label{eq:cert_eb_main}
\end{equation}
Thus, under graph-smooth ground truth, large $E_{\struct}(\hat{Y}_b,A)$ certifies non-negligible $e_b$ (not an exact MSE predictor). Appendix~\ref{sec:appendix_theory} links score ordering to selective risk for a fixed acceptance set.

\subsection{Training objectives}

The auxiliary losses play complementary roles: $\graphloss$ learns \emph{what relational consistency looks like} on ground-truth futures, whereas $\alignloss$ teaches $E_{\struct}(\hat{Y},A)$ \emph{when structural deviation predicts forecast error}.

\paragraph{Forecasting loss.}
\begin{equation}
\forecast = \frac{1}{B}\sum_{b=1}^{B} e_b.
\label{eq:forecast_loss}
\end{equation}

\paragraph{Error-weighted graph regularization.}
Let $Y$ be the ground-truth future and $\hat{Y}$ the model prediction. For each batch sample $b$, let $e_b$ be the per-sample MSE from Eq.~\eqref{eq:mse} and $\bar{e}=\frac{1}{B}\sum_b e_b$. We use normalized weights $w_b=\mathrm{clip}(e_b/\bar{e},\,0.1,\,10)$. With learned $A$, we compute structural energy on $Y$ (same formula as Eq.~\eqref{eq:estruct} with $Y$ in place of $\hat{Y}$) and define
\begin{equation}
\graphloss = \frac{1}{B}\sum_{b=1}^{B} w_b\, E_{\struct}(Y_b, A).
\label{eq:graph_reg}
\end{equation}
Weighting by $w_b$ emphasizes structural patterns in difficult forecasting regions, so $A$ is shaped more strongly where selective decisions matter. $\graphloss$ is computed on $Y$ (not $\hat{Y}$) to learn nominal dependencies.

\paragraph{Alignment loss.}
To encourage $E_{\struct}(\hat{Y},A)$ to preserve the risk ordering used in selective prediction, we optimize a differentiable score--risk association surrogate: batch Pearson correlation between $s_b=E_{\struct}(\hat{Y}_b,A)$ and targets $t_b=\log(1{+}e_b)$ (or $t_b=e_b$). Concretely, $\alignloss = -\operatorname{Corr}(s,t)$ (Eq.~\eqref{eq:pearson_batch}, Appendix~\ref{sec:appendix_alignment_impl}); this does not directly optimize permutation agreement, but is consistent with oracle ranking under approximate score--risk alignment (Proposition~\ref{prop:pairwise_stability}, Appendix~\ref{sec:appendix_theory_alignment}). The joint objective is
\begin{equation}
\totalloss
=
\forecast
+
\lambda_g \graphloss
+
\lambda_a \alignloss,
\label{eq:total_loss}
\end{equation}
with hyperparameters $\lambda_g,\lambda_a$ (defaults in our implementation: $\lambda_g{=}0.1$, $\lambda_a{=}0.05$). Appendix~\ref{sec:appendix_theory} explains when $E_{\struct}$ can track risk under graph smoothness; it does not imply universally optimal selective ranking.

\subsection{\texorpdfstring{\tem{}}{TEM} baseline (prior work)}

We instantiate \tem{} following \citet{brusokas2025time}: an EBM on decoder trajectories trained jointly with $f_\theta$ and the graph head (same splits and early stopping). At inference, $\Etem$ scores each forecast; graph and \tem{} gates use raw energies under identical rules (Section~\ref{sec:experiments}).

\section{Experiments}
\label{sec:experiments}

We compare graph-only $E_{\struct}$ and \tem{}-only $\Etem$ gates (raw per-example scores, no z-score before sorting). \textbf{Protocol~A} fits test-split quantile thresholds (ranking diagnostic; Table~\ref{tab:selective_mse_main}). \textbf{Protocol~B} calibrates on validation and applies thresholds to test (deployable; Table~\ref{tab:selective_mse_vali_test}, Appendix~\ref{sec:appendix_supp_ablation}). Simple instance-level baselines (MC dropout, an error predictor) are in Appendix~\ref{sec:appendix_additional_baselines}; graph vs.\ \tem{} remains the primary axis comparison because both gates are trained in the same SASF pipeline under identical selective rules.

\paragraph{Why graph-only vs.\ \tem{}-only (not fusion).}
We report separate gates as the primary comparison because our goal is to test whether relational consistency adds an \emph{independent} abstention signal beyond instance-level energy. A fused score may improve practical performance but would confound this question by blending two axes whose contributions we aim to isolate. Score fusion (e.g., validation z-scoring before blending) is supported in the repository as an optional downstream extension, not as the core empirical claim.

\subsection{Datasets and backbones}

Seven multivariate long-horizon benchmarks (ETTh1/2, ETTm1/2, Electricity, Traffic, Weather) with standard splits. Backbones: Autoformer, FEDformer, PatchTST, TimesNet \citep{wu2021autoformer,zhou2022fedformer,nie2022time,wu2022timesnet}. The graph head and \tem{} branch train jointly with $\lambda_g{=}0.1$, $\lambda_a{=}0.05$, top-$k{=}10$ when $D{>}10$. The graph head is lightweight in parameter count (embedding matrix plus sparsified $A$); top-$k$ row sparsity limits relational scoring to a small neighborhood per variable, so runtime remains dominated by the forecaster backbone (Appendix~\ref{sec:appendix_compute}).

\subsection{Protocols and metrics}

\paragraph{Protocol~B (deployable).}
Thresholds are validation $\rho$-quantiles of each score, applied to test without re-calibration; selective MSE is mean $e_b$ (Eq.~\eqref{eq:mse}) over accepted points. \textit{N/A} marks zero test acceptances under split shift.

\paragraph{Protocol~A (diagnostic; Table~\ref{tab:selective_mse_main}).}
Test-split quantiles hold nominal coverage by construction, isolating ranking under matched retention. The Full column is full-coverage test MSE. Tables~\ref{tab:selective_mse_main}--\ref{tab:ablation2_main} use raw \texttt{ebm\_only}/\texttt{graph\_only} energies (fusion variants in the repository).

\paragraph{Ranking view.}
Under Protocol~A, at nominal coverage $\rho$, the $\lfloor \rho B \rfloor$ test samples with lowest abstention scores are retained by construction (Eq.~\eqref{eq:objective}). Under Protocol~B, the threshold is the validation $\rho$-quantile of each score, applied to test without re-calibration; realized test coverage can differ from $\rho$ under split shift (including zero acceptances, \textit{N/A}). In both cases, selective risk on the accepted set depends on how well scores separate lower- from higher-error samples $e_b$ (Eq.~\eqref{eq:mse}). For a fixed realized acceptance set of size $k$, sorting by the true error oracle minimizes empirical selective risk (Theorem~\ref{thm:oracle_ranking}, Appendix~\ref{sec:appendix_theory_ranking}). Training encourages $E_{\struct}$ to align with error proxies via $\alignloss$ (Section~\ref{sec:method}); mis-ranked acceptances bound excess risk at size $k$ (Theorem~\ref{thm:excess_risk}).

\noindent\textbf{I\%}$\uparrow$ compares graph vs.\ \tem{} at the same nominal target $\rho$: $(\mathrm{TEM}-\mathrm{Graph})/\mathrm{TEM}\times 100$. \textbf{S\%}$\uparrow$ compares graph vs.\ full coverage: $(\mathrm{Full}-\mathrm{Graph}_{\rho})/\mathrm{Full}\times 100$. Table~\ref{tab:ablation2_main} uses the same relative-improvement convention $(1-\mathrm{MSE}_{\mathrm{sel}}/\mathrm{MSE}_{\mathrm{orig}})\times 100$ with \texttt{tau\_mode}=\texttt{quantile}; Full (A0) aligns with \textbf{S\%}$\uparrow$ in Table~\ref{tab:selective_mse_main}. Appendix~\ref{sec:appendix_supp_ablation} (Table~\ref{tab:ablation_main}) reports a complementary validation$\to$test threshold protocol that need not match every cell numerically.

\subsection{TEM vs.\ graph energy (main results)}
\label{sec:tem_vs_graph}

Table~\ref{tab:selective_mse_main} summarizes \emph{Protocol~A} selective MSE for \tem{}-only vs.\ graph-only gates; deployable Protocol~B results are in Table~\ref{tab:selective_mse_vali_test}.

\paragraph{Findings.}
\textbf{(F1) Coupling.} Graph-only gates improve most on Weather, Electricity, and Traffic---benchmarks where cross-variable structure appears more informative for abstention in our setting---often on Autoformer, FEDformer, and PatchTST (positive \textbf{I\%}/\textbf{S\%} in Table~\ref{tab:selective_mse_main}).
\textbf{(F2) Complementarity.} TimesNet and several ETT splits favor \tem{}; these failures are informative: relational abstention targets structural risk and should not dominate when errors are primarily marginal.
\textbf{(F3) Transfer.} Protocol~B trends align with Protocol~A on overlapping cells but include split-shift \textit{N/A} cells (Appendix~\ref{sec:appendix_supp_ablation}). Risk--coverage curves: Appendix~\ref{sec:appendix_backbone_dataset_curves}.

% Unified selective MSE: TEM vs Graph (\texttt{tab:selective\_mse\_main}).
\begin{table*}[!t]
  \centering
  \footnotesize
  \setlength{\tabcolsep}{3.8pt}%
  \renewcommand{\arraystretch}{1.12}%
  \caption{\textbf{(Protocol~A; ranking diagnostic.)} Test selective prediction MSE across datasets and backbones (lower is better). Thresholds are \textbf{test-split} empirical $\rho$-quantiles of each raw score (\tem{} vs.\ structural), not validation-calibrated deployment thresholds (Protocol~B: Table~\ref{tab:selective_mse_vali_test}). Graph is bolded when lower than TEM at the same $\rho$. \textbf{I\%}$\uparrow$: $(\text{TEM}-\text{Graph})/\text{TEM}\times 100$. \textbf{S\%}$\uparrow$: $(\text{Full}-\text{Graph}_{\rho})/\text{Full}\times 100$. Positive is better for both.}
  \label{tab:selective_mse_main}
  \resizebox{\textwidth}{!}{%
  \begin{tabular}{@{}ll@{\hspace{3pt}}rrrr@{\hspace{9pt}}rrrr@{\hspace{9pt}}rrrr@{\hspace{11pt}}r@{}}
  \toprule
  Dataset & Backbone & \multicolumn{4}{c}{$\rho=0.5$} & \multicolumn{4}{c}{$\rho=0.7$} & \multicolumn{4}{c}{$\rho=0.9$} & \multirow{2}{*}{\textbf{Full}} \\
  \cmidrule(lr){3-6} \cmidrule(lr){7-10} \cmidrule(lr){11-14}
   &  & TEM & Graph & \textbf{I\%}$\uparrow$ & \textbf{S\%}$\uparrow$ & TEM & Graph & \textbf{I\%}$\uparrow$ & \textbf{S\%}$\uparrow$ & TEM & Graph & \textbf{I\%}$\uparrow$ & \textbf{S\%}$\uparrow$ \\
  \midrule
  ETTh1 & Autoformer & 0.4453 & \textbf{0.4209} & +5.5 & +4.1 & 0.4403 & \textbf{0.4091} & +7.1 & +6.8 & 0.4388 & \textbf{0.4250} & +3.1 & +3.1 & 0.4388 \\
   & FEDformer & 0.4030 & \textbf{0.3893} & +3.4 & +3.9 & 0.4063 & \textbf{0.3828} & +5.8 & +5.5 & 0.4051 & \textbf{0.3956} & +2.3 & +2.3 & 0.4049 \\
   & PatchTST & 0.3837 & \textbf{0.3547} & +7.6 & +7.6 & 0.3824 & \textbf{0.3650} & +4.6 & +4.9 & 0.3839 & \textbf{0.3780} & +1.5 & +1.5 & 0.3838 \\
   & TimesNet & 0.3716 & \textbf{0.3484} & +6.2 & +10.2 & 0.3740 & \textbf{0.3685} & +1.5 & +5.0 & 0.3862 & 0.3916 & -1.4 & -1.0 & 0.3879 \\
  ETTh2 & Autoformer & 0.2587 & \textbf{0.2448} & +5.4 & +18.7 & 0.2665 & 0.2685 & -0.8 & +10.8 & 0.2882 & 0.2907 & -0.9 & +3.5 & 0.3011 \\
   & FEDformer & 0.2703 & \textbf{0.2452} & +9.3 & +22.4 & 0.3025 & \textbf{0.2603} & +14.0 & +17.6 & 0.3143 & \textbf{0.2924} & +7.0 & +7.4 & 0.3159 \\
   & PatchTST & 0.2597 & \textbf{0.1951} & +24.9 & +24.1 & 0.2595 & \textbf{0.2021} & +22.1 & +21.3 & 0.2615 & \textbf{0.2358} & +9.8 & +8.2 & 0.2569 \\
   & TimesNet & 0.2493 & 0.3313 & -32.9 & -26.0 & 0.2534 & 0.3173 & -25.2 & -20.7 & 0.2552 & 0.2828 & -10.8 & -7.6 & 0.2629 \\
  ETTm1 & Autoformer & 0.4509 & \textbf{0.3928} & +12.9 & +11.0 & 0.4435 & \textbf{0.4342} & +2.1 & +1.6 & 0.4390 & \textbf{0.4374} & +0.4 & +0.8 & 0.4411 \\
   & FEDformer & 0.4724 & \textbf{0.3980} & +15.7 & +11.7 & 0.4678 & \textbf{0.4326} & +7.5 & +4.1 & 0.4615 & \textbf{0.4448} & +3.6 & +1.4 & 0.4509 \\
   & PatchTST & 0.4003 & \textbf{0.3262} & +18.5 & +9.4 & 0.3816 & \textbf{0.3450} & +9.6 & +4.2 & 0.3657 & \textbf{0.3548} & +3.0 & +1.5 & 0.3601 \\
   & TimesNet & 0.2824 & 0.3171 & -12.3 & -1.4 & 0.2877 & 0.3219 & -11.9 & -2.9 & 0.2969 & 0.3168 & -6.7 & -1.3 & 0.3127 \\
  ETTm2 & Autoformer & 0.1844 & \textbf{0.1736} & +5.9 & -1.0 & 0.1762 & \textbf{0.1582} & +10.2 & +8.0 & 0.1735 & \textbf{0.1522} & +12.3 & +11.4 & 0.1719 \\
   & FEDformer & 0.1363 & 0.1593 & -16.9 & +8.7 & 0.1342 & 0.1495 & -11.4 & +14.3 & 0.1688 & \textbf{0.1575} & +6.7 & +9.7 & 0.1744 \\
   & PatchTST & 0.1516 & 0.1547 & -2.0 & +2.5 & 0.1469 & 0.1471 & -0.1 & +7.3 & 0.1530 & \textbf{0.1419} & +7.3 & +10.5 & 0.1586 \\
   & TimesNet & 0.1299 & 0.1474 & -13.5 & -3.4 & 0.1341 & \textbf{0.1301} & +3.0 & +8.8 & 0.1370 & \textbf{0.1316} & +3.9 & +7.7 & 0.1426 \\
  Electricity & Autoformer & 0.2925 & \textbf{0.2633} & +10.0 & +9.8 & 0.2933 & \textbf{0.2762} & +5.8 & +5.4 & 0.2924 & \textbf{0.2863} & +2.1 & +2.0 & 0.2920 \\
   & FEDformer & 0.1890 & \textbf{0.1601} & +15.3 & +15.8 & 0.1892 & \textbf{0.1683} & +11.0 & +11.5 & 0.1897 & \textbf{0.1859} & +2.0 & +2.2 & 0.1901 \\
   & PatchTST & 0.1978 & \textbf{0.1724} & +12.8 & +15.1 & 0.2015 & \textbf{0.1804} & +10.5 & +11.1 & 0.2026 & \textbf{0.1979} & +2.3 & +2.5 & 0.2030 \\
   & TimesNet & 0.1753 & \textbf{0.1311} & +25.2 & +20.3 & 0.1694 & \textbf{0.1398} & +17.5 & +15.0 & 0.1654 & \textbf{0.1494} & +9.7 & +9.2 & 0.1645 \\
  Traffic & Autoformer & 1.0117 & \textbf{0.9479} & +6.3 & +3.2 & 0.9971 & \textbf{0.9532} & +4.4 & +2.6 & 0.9853 & \textbf{0.9608} & +2.5 & +1.9 & 0.9791 \\
   & FEDformer & 0.6945 & 0.7003 & -0.8 & +0.8 & 0.6982 & \textbf{0.6948} & +0.5 & +1.6 & 0.7046 & \textbf{0.6998} & +0.7 & +0.9 & 0.7061 \\
   & PatchTST & 0.6260 & \textbf{0.6087} & +2.8 & +3.1 & 0.6132 & \textbf{0.5918} & +3.5 & +0.3 & 0.5994 & \textbf{0.5886} & +1.8 & +0.3 & 0.5903 \\
   & TimesNet & 0.5922 & \textbf{0.5464} & +7.7 & +6.4 & 0.5866 & \textbf{0.5513} & +6.0 & +5.6 & 0.5866 & \textbf{0.5695} & +2.9 & +2.5 & 0.5839 \\
  Weather & Autoformer & 0.2675 & \textbf{0.1896} & +29.1 & +20.5 & 0.2521 & \textbf{0.2012} & +20.2 & +15.7 & 0.2414 & \textbf{0.2236} & +7.4 & +6.3 & 0.2386 \\
   & FEDformer & 0.3210 & \textbf{0.2348} & +26.9 & +22.8 & 0.3130 & \textbf{0.2561} & +18.2 & +15.8 & 0.3090 & \textbf{0.2829} & +8.4 & +6.9 & 0.3040 \\
   & PatchTST & 0.1783 & \textbf{0.1338} & +24.9 & +20.7 & 0.1727 & \textbf{0.1381} & +20.0 & +18.1 & 0.1701 & \textbf{0.1548} & +9.0 & +8.2 & 0.1687 \\
   & TimesNet & 0.1278 & \textbf{0.1133} & +11.3 & +17.7 & 0.1316 & \textbf{0.1216} & +7.6 & +11.7 & 0.1343 & \textbf{0.1256} & +6.5 & +8.8 & 0.1377 \\
  \bottomrule
  \end{tabular}
  }
\end{table*}

\begin{figure*}[!t]
  \centering
  \setlength{\casefigH}{0.27\textheight}
  \begin{minipage}[t]{0.54\textwidth}
    \centering
    \includegraphics[width=0.98\linewidth,height=\casefigH,keepaspectratio]{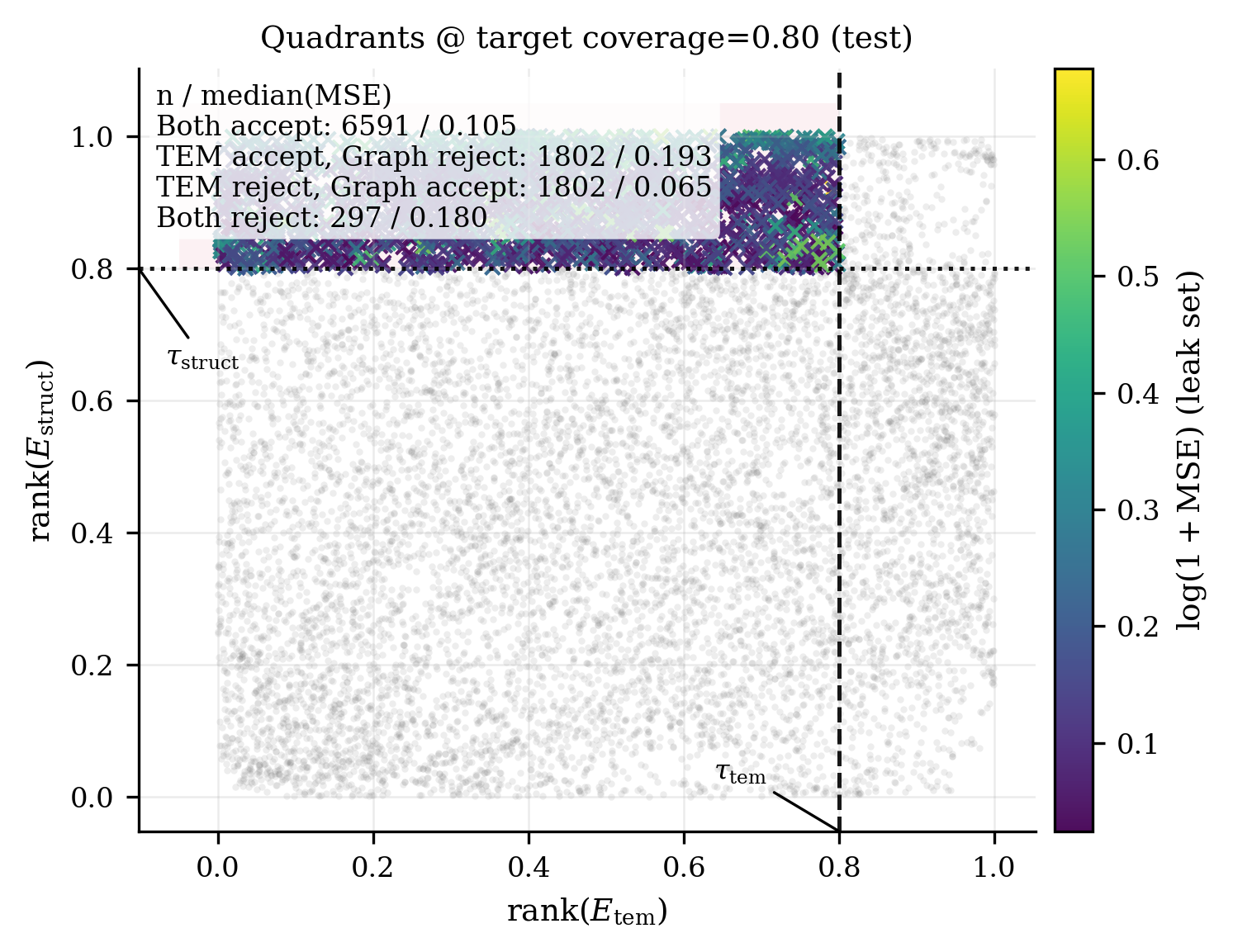}
  \end{minipage}\hfill
  \begin{minipage}[t]{0.44\textwidth}
    \centering
    \includegraphics[width=0.98\linewidth,height=\casefigH,keepaspectratio]{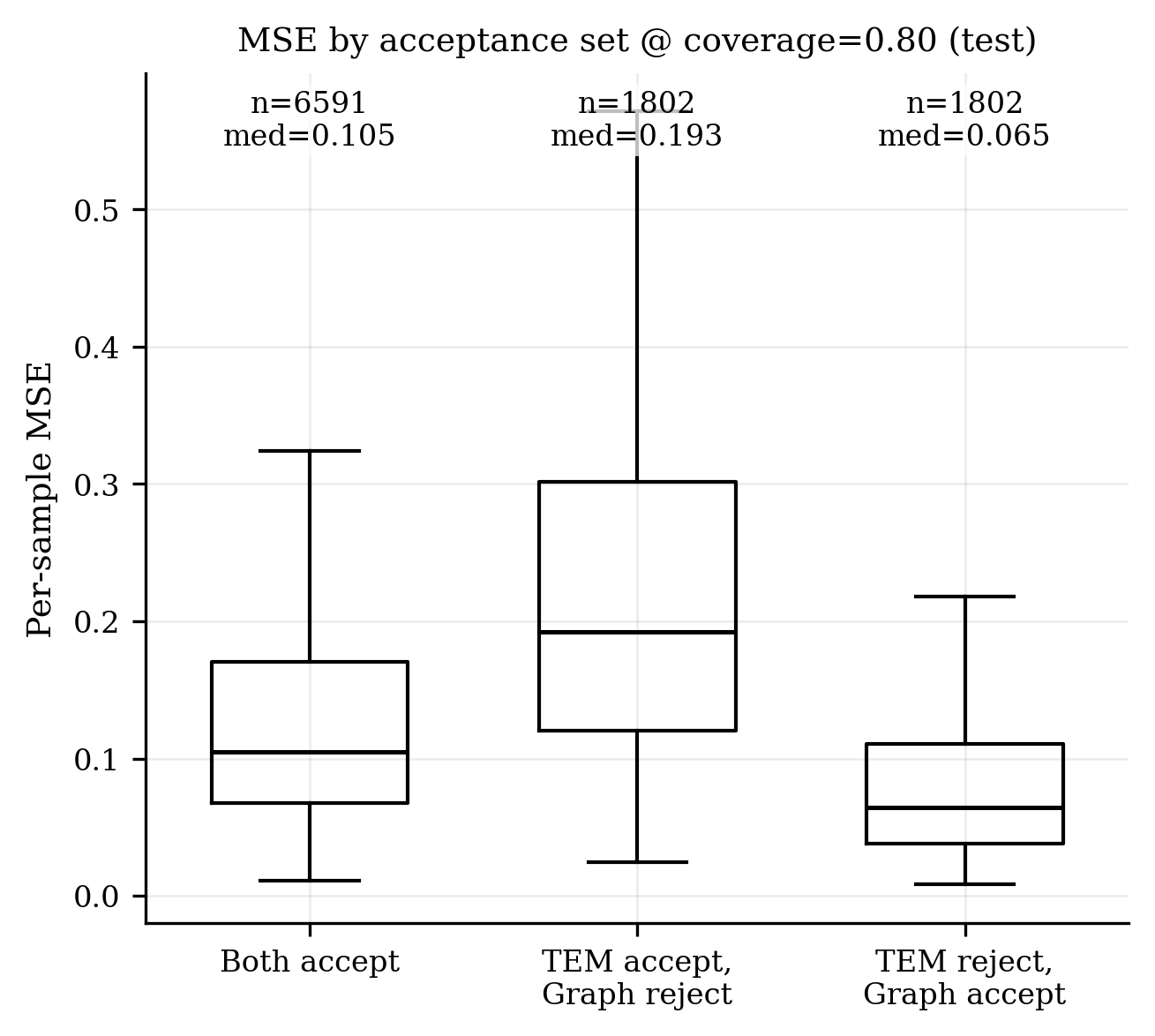}
  \end{minipage}
  \caption{Case study (coverage $0.80$). Percentile ranks on the test split (lower $\Rightarrow$ more likely accepted). Dashed lines mark $80\%$ nominal coverage via test-split score quantiles under the same protocol as Table~\ref{tab:selective_mse_main}. Left: ranks of $\Etem$ vs.\ $E_{\struct}$; upper-left marks \emph{relational misses} (\tem{} accepts, graph rejects). Right: test MSE by acceptance region; relational misses skew higher-error than mutual acceptance.}
  \label{fig:case_quadrants_boxplot}
\end{figure*}

\subsection{Case study: when TEM misses structural failures}
\label{sec:case_study}

To make the qualitative difference between instance-level and relational abstention explicit, Figure~\ref{fig:case_quadrants_boxplot} visualizes acceptance disagreements at $\rho=0.80$ on the test split (chosen only for visual separation in the rank--rank scatter, \emph{not} because it is a reported operating point; tabulated coverages remain $\rho\in\{0.5,0.7,0.9\}$). We use percentile ranks of the two scores (lower score $\Rightarrow$ more likely accepted) and dashed lines at $80\%$ nominal retention under the same Protocol~A quantile logic as Table~\ref{tab:selective_mse_main}.

We define a \emph{relational miss} as a sample $\Etem$ would accept but $E_{\struct}$ rejects. Figure~\ref{fig:case_quadrants_boxplot} and Figure~\ref{fig:case_heatmap} show relational misses carry higher test MSE than mutual acceptance (qualitative support for Section~\ref{sec:method}).

\begin{figure}[!t]
  \centering
  \includegraphics[width=0.97\linewidth]{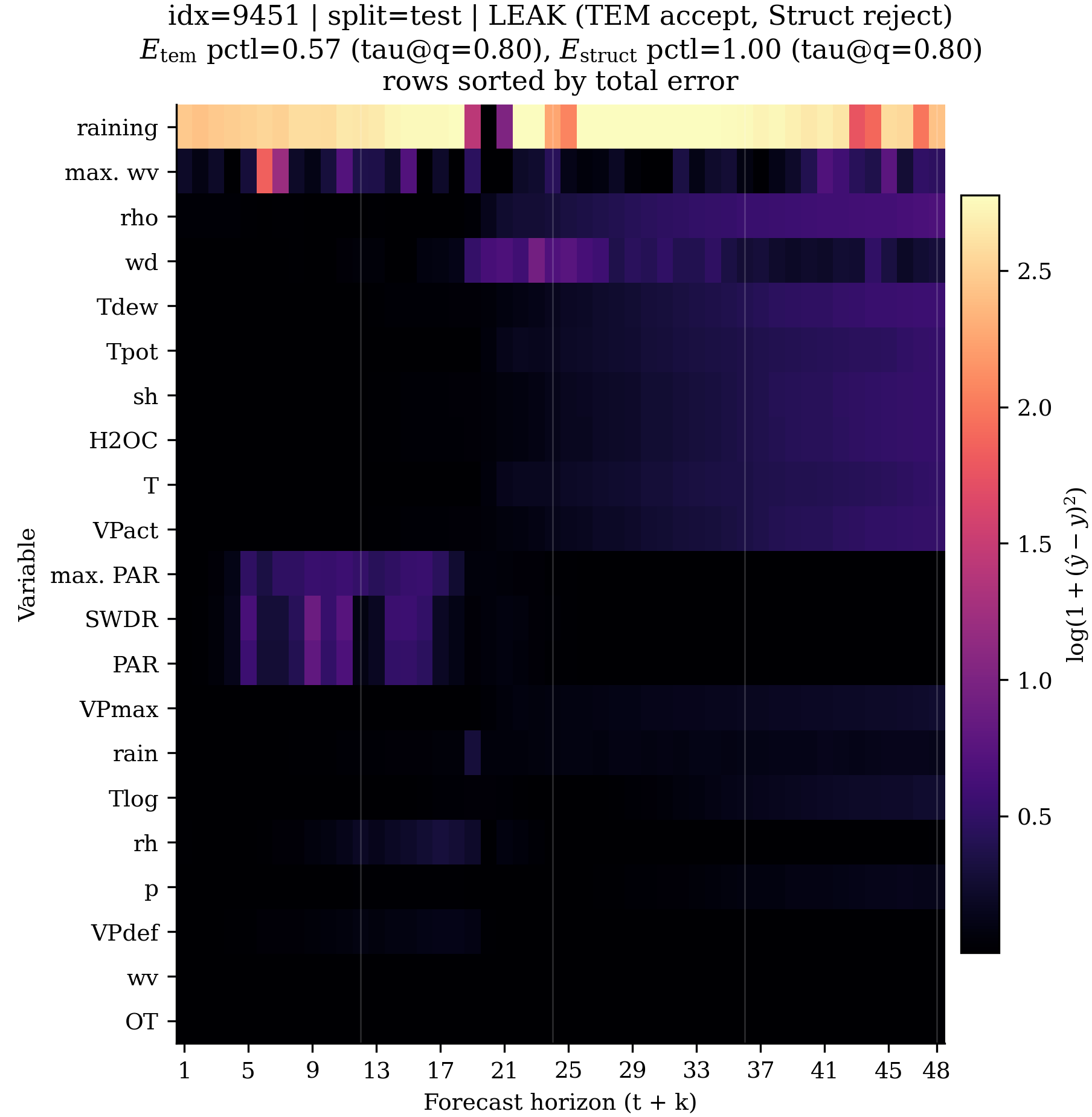}
  \caption{A representative relational miss (test index 9451): \tem{} accepts at $0.80$ target coverage while the structural gate rejects. Heatmap: channel-wise squared error over the horizon (log-scaled).}
  \label{fig:case_heatmap}
\end{figure}

\subsection{Ablations}
\label{sec:ablations}

On Weather, Electricity, and ETTh2 (PatchTST, Autoformer), A0 is full SASF; A1 drops alignment ($\lambda_a{=}0$); A2 uses uniform $w_b$; A3 fixes $A$ to a correlation graph. A1 tests whether $E_{\struct}$ ranks risk without explicit alignment; A2 whether error weighting shapes $A$; A3 whether adaptation beats a static graph. Table~\ref{tab:ablation2_main} (Protocol~A): A1/A2 usually hurt versus A0; A3 is often worse than A0. Validation$\to$test ablations: Appendix~\ref{sec:appendix_supp_ablation}.

% Main-paper graph ablation (\texttt{tau\_mode}=quantile).
\begin{table*}[!t]
  \centering
  \footnotesize
  \setlength{\tabcolsep}{3.5pt}
  \renewcommand{\arraystretch}{1.05}
  \caption{%
    Graph-only ablation (\textbf{PatchTST}, \textbf{Autoformer}; \texttt{tau\_mode}=\texttt{quantile}, Protocol~A).
    Test relative improvement $(1-\mathrm{MSE}_{\mathrm{sel}}/\mathrm{MSE}_{\mathrm{orig}})\times 100$ at $c\in\{0.5,0.7,0.9\}$ on Weather, Electricity, ETTh2.
    \textbf{A0} full; \textbf{A1} w/o alignment; \textbf{A2} w/o error weighting; \textbf{A3} correlation-graph gate.
    \textbf{Full (A0)} matches Table~\ref{tab:selective_mse_main} \textbf{S\%}$\uparrow$; bold: best per column within each backbone block.
  }
  \label{tab:ablation2_main}
  \begin{tabular}{@{}ll *{9}{r}@{}}
    \toprule
    & & \multicolumn{3}{c}{\textbf{Weather}} & \multicolumn{3}{c}{\textbf{Electricity}} & \multicolumn{3}{c}{\textbf{ETTh2}} \\
    \cmidrule(lr){3-5} \cmidrule(lr){6-8} \cmidrule(lr){9-11}
    \textbf{Backbone} & \textbf{Ablation} & $0.5$ & $0.7$ & $0.9$ & $0.5$ & $0.7$ & $0.9$ & $0.5$ & $0.7$ & $0.9$ \\
    \midrule
    \multirow{4}{*}{PatchTST}
      & Full (A0) & \textbf{20.70} & \textbf{18.10} & \textbf{8.20} & \textbf{15.10} & \textbf{11.10} & \textbf{2.50} & \textbf{24.10} & \textbf{21.30} & \textbf{8.20} \\
      & w/o align (A1) & -8.13 & -2.04 & 0.78 & 13.30 & 9.74 & 2.12 & -4.76 & 1.70 & 4.45 \\
      & w/o err.-wt (A2) & 10.16 & 12.51 & 7.88 & 12.87 & 5.54 & 0.80 & 9.17 & 9.15 & 5.38 \\
      & Corr.\ gate (A3) & -7.30 & -3.97 & -0.81 & 13.61 & 10.19 & 1.75 & 4.29 & 6.86 & 1.88 \\
    \midrule
    \multirow{4}{*}{Autoformer}
      & Full (A0) & \textbf{20.50} & \textbf{15.70} & \textbf{6.30} & \textbf{9.80} & \textbf{5.40} & \textbf{2.00} & \textbf{18.70} & \textbf{10.80} & \textbf{3.50} \\
      & w/o align (A1) & 1.17 & 5.21 & 3.10 & 8.33 & 4.62 & 1.78 & -0.13 & -0.40 & 0.15 \\
      & w/o err.-wt (A2) & 18.88 & 15.22 & 6.14 & 7.49 & 4.92 & 0.31 & 16.11 & 7.19 & 0.36 \\
      & Corr.\ gate (A3) & 0.93 & 2.59 & 1.44 & 8.97 & 3.40 & 1.78 & -5.73 & 0.72 & 0.80 \\
    \bottomrule
  \end{tabular}
\end{table*}

\subsection{When does relational abstention help?}
\label{sec:when_helps}

\paragraph{Three regimes.}
We interpret backbone--dataset variation through three regimes: (i)~\emph{relational failure dominates}---high error aligns with violations of stable cross-variable structure (relational misses; Section~\ref{sec:case_study}); (ii)~\emph{marginal failure dominates}---error is mainly per-channel or trajectory-level, so instance-level $\Etem$ ranks risk well; (iii)~\emph{global graph mismatch}---a single training-time $A$ poorly matches the test regime (weak or shifting coupling), so $E_{\struct}$ is less discriminative. (F1)--(F3) map to these regimes rather than to uniform superiority of either gate.

SASF is most useful in regime~(i): $E_{\struct}$ targets a risk axis that $\Etem$ does not explicitly test, consistent with gains on Weather, Electricity, and Traffic. In regimes~(ii)--(iii), $\Etem$ often wins (e.g., TimesNet and several ETT splits in F2); we interpret these primarily as informative scope limits rather than evidence for universal superiority of either gate.

\paragraph{Link to analysis.}
Theorem~\ref{thm:error_certificate} formalizes regime~(i): when $E_{\struct}(Y,A)$ is small, large $E_{\struct}(\hat{Y},A)$ certifies non-negligible squared error; the certificate weakens when smoothness fails (regimes~(ii)--(iii), Appendix~\ref{sec:appendix_theory_scope}). Heterogeneous backbone--dataset results therefore reflect how much selective risk lies on the relational axis captured by $E_{\struct}$, not an unexpected breakdown of the method.

\paragraph{Practical implications.}
$\Etem$ and $E_{\struct}$ encode different reliability axes; deployments may prefer one gate, switch by regime, or apply fusion (repository) when the dominant failure mode is known---relational abstention complements instance-level scoring rather than replacing it.

\FloatBarrier
\subsection{Threats to validity}
\label{sec:threats_validity}

\paragraph{Statistical robustness.}
Table~\ref{tab:selective_mse_main} (Protocol~A) uses one checkpoint per cell (seed~2024) and isolates ranking under matched test coverage; it is not a deployable threshold-transfer readout. To probe seed sensitivity under Protocol~B, we rerun validation$\to$test selection with seeds $\{0,1,2024\}$ on ETTh1/ETTh2/Electricity for PatchTST and Autoformer, excluding partial cells rather than imputing them (Appendix~\ref{sec:appendix_multiseed}). Table~\ref{tab:multiseed_protocol_b} aligns both backbones on the same three complete $(\text{dataset},\rho)$ settings; mean $I\%\uparrow$ averages cell-level means (each cell over seeds), and $\pm$ is the std across those cell means.

% Multiseed Protocol B — aligned subset (same 3 cells per backbone).
\begin{table}[!t]
  \centering
  \footnotesize
  \setlength{\tabcolsep}{3pt}
  \renewcommand{\arraystretch}{1.15}
  \caption{%
    \textbf{Multiseed deployable Protocol~B (aligned subset).}
    Same three complete $(\text{dataset},\rho)$ cells for both backbones: ETTh1 ($\rho{=}0.7,0.9$) and ETTh2 ($\rho{=}0.9$), with seeds $\{0,1,2024\}$ (partial cells excluded).%
  }
  \label{tab:multiseed_protocol_b}
  \resizebox{\columnwidth}{!}{%
  \begin{tabular}{@{}lcccc@{}}
    \toprule
    \textbf{Backbone} & \textbf{Cells} & \textbf{Seeds} & \textbf{Wins} & \textbf{Mean $I\%\uparrow$} \\
    \midrule
    PatchTST   & 3 & 9 & 9/9 & $+10.6 \pm 12.8$ \\
    Autoformer & 3 & 9 & 7/9 & $+7.1 \pm 17.3$ \\
    \bottomrule
  \end{tabular}%
  }
\end{table}

\noindent On this aligned subset, graph selective MSE wins all nine PatchTST seed runs (9/9) and seven of nine Autoformer runs (7/9; one-sided binomial sign test $p{\approx}0.09$). Extended PatchTST complete cells ($n{=}21$ seed runs, all graph wins) yield binomial and Wilcoxon tests with $p{\ll}10^{-5}$ (subset; not the full main-table grid; details in Appendix~\ref{sec:appendix_multiseed}). Small \textbf{I\%}/\textbf{S\%} gaps elsewhere in Table~\ref{tab:selective_mse_main} should still be read cautiously; we emphasize directional patterns (F1--F3) rather than uniform gains on every backbone--dataset pair.

\paragraph{Structural assumptions.}
Relational gating assumes that a global, unsigned smoothness prior over variables is a useful abstention proxy. This may hold unevenly across datasets and backbones (Section~\ref{sec:when_helps}); it does not model lags, signed coupling, or sample-specific graphs.

\paragraph{Deployment transfer.}
Protocol~B applies validation thresholds to test without recalibration. Validation--test score shift can yield zero test acceptances (\textit{N/A} in Appendix~\ref{sec:appendix_supp_ablation}), so deployable selective MSE is undefined in those cells. This limits threshold \emph{transfer} under Protocol~B; Protocol~A ranking diagnostics are unaffected.

\section{Conclusion}
\label{sec:conclusion}

We studied \emph{structure-aware selective forecasting}: relational consistency, operationalized by $E_{\struct}$, complements instance-level \tem{} under explicit coverage constraints. Gains concentrate where relational structure is informative (Section~\ref{sec:when_helps}); negative cells mark scope limits rather than a universal replacement for instance scores. Multivariate selective forecasting should ask both whether trajectories are plausible and whether they are relationally consistent. Limitations (protocols, seeds, deployment transfer) are in Section~\ref{sec:threats_validity}.

\FloatBarrier

\subsection*{AI Use Statement}

Generative AI tools were used only for English language editing, readability, and LaTeX/formatting assistance on the manuscript. They were not used for model design, training, or the experimental pipeline; for theoretical results, proofs, or mathematical claims; or for producing or altering primary experimental numbers reported in the paper.

All AI-assisted wording and formatting suggestions were reviewed, verified, and edited by the authors. The authors take full responsibility for the final content of this work, including text, claims, and artifacts.

\bibliography{references}

@article{geifman2017selective,
  title={Selective classification for deep neural networks},
  author={Geifman, Yonatan and El-Yaniv, Ran},
  journal={Advances in neural information processing systems},
  volume={30},
  year={2017}
}

@article{wu2021autoformer,
  title={Autoformer: Decomposition transformers with auto-correlation for long-term series forecasting},
  author={Wu, Haixu and Xu, Jiehui and Wang, Jianmin and Long, Mingsheng},
  journal={Advances in neural information processing systems},
  volume={34},
  pages={22419--22430},
  year={2021}
}

@inproceedings{zhou2022fedformer,
  title={Fedformer: Frequency enhanced decomposed transformer for long-term series forecasting},
  author={Zhou, Tian and Ma, Ziqing and Wen, Qingsong and Wang, Xue and Sun, Liang and Jin, Rong},
  booktitle={International conference on machine learning},
  pages={27268--27286},
  year={2022},
  organization={PMLR}
}

@article{nie2022time,
  title={A time series is worth 64 words: Long-term forecasting with transformers. arXiv 2022},
  author={Nie, Yuqi and Nguyen, Nam H and Sinthong, Phanwadee and Kalagnanam, Jayant},
  journal={arXiv preprint arXiv:2211.14730},
  year={2022}
}

@article{wu2022timesnet,
  title={Timesnet: Temporal 2d-variation modeling for general time series analysis},
  author={Wu, Haixu and Hu, Tengge and Liu, Yong and Zhou, Hang and Wang, Jianmin and Long, Mingsheng},
  journal={arXiv preprint arXiv:2210.02186},
  year={2022}
}

@article{lecun2006tutorial,
  title={A tutorial on energy-based learning},
  author={LeCun, Yann and Chopra, Sumit and Hadsell, Raia and Ranzato, M and Huang, Fujie and others},
  journal={Predicting structured data},
  volume={1},
  number={0},
  year={2006}
}

@inproceedings{gal2016dropout,
  title={Dropout as a {B}ayesian Approximation: Representing Model Uncertainty in Deep Learning},
  author={Gal, Yarin and Ghahramani, Zoubin},
  booktitle={International Conference on Machine Learning (ICML)},
  year={2016}
}

@inproceedings{lakshminarayanan2017simple,
  title={Simple and Scalable Predictive Uncertainty Estimation using Deep Ensembles},
  author={Lakshminarayanan, Balaji and Pritzel, Alexander and Blundell, Charles},
  booktitle={Advances in Neural Information Processing Systems (NeurIPS)},
  volume={30},
  year={2017}
}

@article{gneiting2007strictly,
  title={Strictly Proper Scoring Rules, Prediction, and Estimation},
  author={Gneiting, Tilmann and Raftery, Adrian E.},
  journal={Journal of the American Statistical Association},
  volume={102},
  number={477},
  pages={359--378},
  year={2007},
  doi={10.1198/016214506000001437}
}

@book{vovk2005algorithmic,
  title={Algorithmic Learning in a Random World},
  author={Vovk, Vladimir and Gammerman, Alexander and Shafer, Glenn},
  publisher={Springer},
  year={2005},
  isbn={9780387001524}
}

@inproceedings{kipf2017semi,
  title={Semi-Supervised Classification with Graph Convolutional Networks},
  author={Kipf, Thomas N. and Welling, Max},
  booktitle={International Conference on Learning Representations (ICLR)},
  year={2017}
}

@article{benidis2022forecasting,
  title={Deep Learning for Time Series Forecasting: Tutorial and Literature Survey},
  author={Benidis, Konstantinos and Rangapuram, Syama Sundar and Flunkert, Valentin and Wang, Bernie and Maddix, Danielle C. and Turkmen, Caner and Gasthaus, Jan and Bohlke-Schneider, Michael and Salinas, David and Stella, Lorenzo and Aubet, Fran{\c{c}}ois-Xavier and Callot, Laurent and Januschowski, Tim},
  journal={ACM Computing Surveys},
  volume={55},
  number={6},
  articleno={121},
  numpages={36},
  year={2022},
  doi={10.1145/3533382}
}

@inproceedings{romano2019conformalized,
  title={Conformalized Quantile Regression},
  author={Romano, Yaniv and Patterson, Evan and Cand{\`e}s, Emmanuel J.},
  booktitle={Advances in Neural Information Processing Systems (NeurIPS)},
  volume={32},
  year={2019}
}

@inproceedings{wu2020connecting,
  title={Connecting the dots: Multivariate time series forecasting with graph neural networks},
  author={Wu, Zonghan and Pan, Shirui and Long, Guodong and Jiang, Jing and Chang, Xiaojun and Zhang, Chengqi},
  booktitle={Proceedings of the 26th ACM SIGKDD international conference on knowledge discovery \& data mining},
  pages={753--763},
  year={2020}
}

@article{liu2022multivariate,
  title={Multivariate time-series forecasting with temporal polynomial graph neural networks},
  author={Liu, Yijing and Liu, Qinxian and Zhang, Jian-Wei and Feng, Haozhe and Wang, Zhongwei and Zhou, Zihan and Chen, Wei},
  journal={Advances in neural information processing systems},
  volume={35},
  pages={19414--19426},
  year={2022}
}

@article{chen2022balanced,
  title={Balanced Graph Structure Learning for Multivariate Time Series Forecasting},
  author={Chen, Weijun and Wang, Yanze and Du, Chengshuo and Jia, Zhenglong and Liu, Feng and Chen, Ran},
  journal={arXiv preprint arXiv:2201.09686},
  year={2022}
}

@article{chen2023multi,
  title={Multi-scale adaptive graph neural network for multivariate time series forecasting},
  author={Chen, Ling and Chen, Donghui and Shang, Zongjiang and Wu, Binqing and Zheng, Cen and Wen, Bo and Zhang, Wei},
  journal={IEEE Transactions on Knowledge and Data Engineering},
  volume={35},
  number={10},
  pages={10748--10761},
  year={2023},
  publisher={IEEE}
}

@article{cai2024forecastgrapher,
  title={Forecastgrapher: Redefining multivariate time series forecasting with graph neural networks},
  author={Cai, Wanlin and Wang, Kun and Wu, Hao and Chen, Xiaoxu and Wu, Yuankai},
  journal={arXiv preprint arXiv:2405.18036},
  year={2024}
}

@article{kim2023hierarchical,
  title={Hierarchical joint graph learning and multivariate time series forecasting},
  author={Kim, Juhyeon and Lee, Hyungeun and Yu, Seungwon and Hwang, Ung and Jung, Wooyeol and Yoon, Kijung},
  journal={IEEE Access},
  volume={11},
  pages={118386--118394},
  year={2023},
  publisher={IEEE}
}

@article{sriramulu2023adaptive,
  title={Adaptive dependency learning graph neural networks},
  author={Sriramulu, Abishek and Fourrier, Nicolas and Bergmeir, Christoph},
  journal={Information Sciences},
  volume={625},
  pages={700--714},
  year={2023},
  publisher={Elsevier}
}

@inproceedings{zhang2024irregular,
  title={Irregular multivariate time series forecasting: A transformable patching graph neural networks approach},
  author={Zhang, Weijia and Yin, Chenlong and Liu, Hao and Zhou, Xiaofang and Xiong, Hui},
  booktitle={Forty-first International Conference on Machine Learning},
  year={2024}
}

@article{brusokas2025time,
  title={The Time-Energy Model: Selective Time-Series Forecasting Using Energy-Based Models},
  author={Brusokas, Jonas and Tirupathi, Seshu and Zhang, Dalin and Pedersen, Torben Bach},
  journal={Transactions on Machine Learning Research},
  year={2025}
}

@inproceedings{zhou2021informer,
  title={Informer: Beyond efficient transformer for long sequence time-series forecasting},
  author={Zhou, Haoyi and Zhang, Shanghang and Peng, Jieqi and Zhang, Shuai and Li, Jianxin and Xiong, Hui and Zhang, Wancai},
  booktitle={Proceedings of the AAAI conference on artificial intelligence},
  volume={35},
  number={12},
  pages={11106--11115},
  year={2021}
}

@inproceedings{zeng2023transformers,
  title={Are transformers effective for time series forecasting?},
  author={Zeng, Ailing and Chen, Muxi and Zhang, Lei and Xu, Qiang},
  booktitle={Proceedings of the AAAI conference on artificial intelligence},
  volume={37},
  number={9},
  pages={11121--11128},
  year={2023}
}

@inproceedings{cortes2016learning,
  title={Learning with rejection},
  author={Cortes, Corinna and DeSalvo, Giulia and Mohri, Mehryar},
  booktitle={International conference on algorithmic learning theory},
  pages={67--82},
  year={2016},
  organization={Springer}
}

@article{angelopoulos2021gentle,
  title={A gentle introduction to conformal prediction and distribution-free uncertainty quantification},
  author={Angelopoulos Anastasios, N and Stephen, Bates},
  journal={arXiv preprint arXiv:2107.07511},
  year={2021}
}

@article{yu2017spatio,
  title={Spatio-temporal graph convolutional networks: A deep learning framework for traffic forecasting},
  author={Yu, Bing and Yin, Haoteng and Zhu, Zhanxing},
  journal={arXiv preprint arXiv:1709.04875},
  year={2017}
}

@article{hyvarinen2005estimation,
  title={Estimation of non-normalized statistical models by score matching.},
  author={Hyv{\"a}rinen, Aapo and Dayan, Peter},
  journal={Journal of Machine Learning Research},
  volume={6},
  number={4},
  year={2005}
}

\clearpage
\onecolumn
\appendix
% Avoid huge vertical stretch when [H] floats leave a short page (global \flushbottom).
\raggedbottom
\setlength{\textfloatsep}{8pt plus 2pt minus 4pt}
\setlength{\intextsep}{8pt plus 2pt minus 4pt}
\setlength{\floatsep}{8pt plus 2pt minus 4pt}
\setlength{\parskip}{0.35\baselineskip}

\section{Supplementary ablation}
\label{sec:appendix_supp_ablation}

\noindent\textbf{Deployable evaluation (Protocol~B).} This appendix is the primary deployment-style readout: thresholds are fit on validation quantiles and applied to test selection (Section~\ref{sec:experiments}). Table~\ref{tab:selective_mse_vali_test} is the full \tem{} vs.\ graph comparison; Table~\ref{tab:ablation_main} gives a complementary ablation under the same validation$\to$test pipeline. Main-text Table~\ref{tab:selective_mse_main} remains Protocol~A (test quantiles) for ranking diagnostics only.

% Appendix ablation: val-split thresholds, test-time selection (supplementary to table_selective_mse_unified.tex).
% Requires: \usepackage{multirow}

\begin{table}[H]
  \centering
  \footnotesize
  \setlength{\tabcolsep}{3.5pt}
  \caption{%
    \textbf{(Supplementary.)}
    Ablation study for \textbf{graph-only} selective inference (\textbf{PatchTST} and \textbf{Autoformer}).
    For each target coverage, we choose energy thresholds on the \textbf{validation} split and apply them to select samples on the \textbf{test} split (complementary to the selective reporting used for Table~\ref{tab:selective_mse_main} in the main text).
    We report \emph{test relative improvement} $(1-\mathrm{MSE}_{\mathrm{sel}}/\mathrm{MSE}_{\mathrm{orig}})\times 100$ at $c\in\{0.5,0.7,0.9\}$.
    Datasets: Weather, Electricity, ETTh2 (multivariate \texttt{M}).
    \textbf{A0}: full model (learned adaptive graph, alignment + error-weighted graph loss);
    \textbf{A1}: $\lambda_a{=}0$ (no alignment);
    \textbf{A2}: uniform graph regularization weights (no error weighting);
    \textbf{A3}: fixed correlation-graph gate (no learned $A$).
    \emph{Bold} values: best relative improvement in each column within each backbone block (ties bolded together).
  }
  \label{tab:ablation_main}
  \begin{tabular}{@{}ll *{9}{r}@{}}
    \toprule
    & & \multicolumn{3}{c}{\textbf{Weather}} & \multicolumn{3}{c}{\textbf{Electricity}} & \multicolumn{3}{c}{\textbf{ETTh2}} \\
    \cmidrule(lr){3-5} \cmidrule(lr){6-8} \cmidrule(lr){9-11}
    \textbf{Backbone} & \textbf{Ablation} & $0.5$ & $0.7$ & $0.9$ & $0.5$ & $0.7$ & $0.9$ & $0.5$ & $0.7$ & $0.9$ \\
    \midrule
    \multirow{4}{*}{PatchTST}
      & Full (A0) & \textbf{12.78} & \textbf{9.73} & \textbf{3.55} & \textbf{28.66} & \textbf{20.50} & \textbf{17.36} & \textbf{38.17} & \textbf{11.27} & \textbf{8.80} \\
      & w/o align (A1) & -9.06 & -2.42 & 0.29 & 8.62 & 9.31 & 13.52 & -7.09 & -2.58 & 2.03 \\
      & w/o err.-wt (A2) & 12.31 & 9.42 & 3.39 & 24.01 & 20.39 & 14.70 & 37.91 & 10.72 & 8.54 \\
      & Corr.\ gate (A3) & -7.51 & -7.30 & -3.28 & 14.60 & 14.16 & 14.27 & 10.87 & -5.78 & -7.65 \\
    \midrule
    \multirow{4}{*}{Autoformer}
      & Full (A0) & \textbf{20.50} & \textbf{15.66} & \textbf{6.87} & \textbf{17.83} & \textbf{14.10} & \textbf{13.94} & \textbf{76.77} & \textbf{66.98} & \textbf{52.25} \\
      & w/o align (A1) & 3.12 & 1.75 & 5.29 & 12.33 & 13.53 & 8.97 & 58.73 & 53.66 & 31.62 \\
      & w/o err.-wt (A2) & 19.77 & 14.48 & 2.66 & 11.55 & 11.59 & 6.45 & 75.56 & 64.38 & 51.98 \\
      & Corr.\ gate (A3) & 3.73 & 3.47 & 2.70 & 16.03 & 13.76 & 12.56 & 26.01 & -17.65 & -8.10 \\
    \bottomrule
  \end{tabular}
\end{table}

\vspace{0.5em}
\noindent\emph{Validation$\to$test thresholding (TEM vs.\ graph).}
Under the same validation-calibrated thresholding protocol, Table~\ref{tab:selective_mse_vali_test} reports \tem{} vs.\ graph selective MSE across all backbones and datasets.
In a few backbone--dataset cases, a validation-derived threshold can yield zero accepted test samples (coverage $=0$) because abstention-score distributions shift across splits; we mark these as N/A since selective MSE and derived I/S metrics are undefined.

% Unified validation->test thresholding table (appendix).
% Thresholds are fit on validation quantiles and applied on test.
% TEM = ebm_only, Graph = graph_only. "---" = missing row in CSV (no test samples under val-derived threshold).
% Preamble: \usepackage{booktabs,graphicx,multirow}
\begin{table}[H]
  \centering
  \footnotesize
  \setlength{\tabcolsep}{3.8pt}%
  \renewcommand{\arraystretch}{1.12}%
  \caption{Validation$\to$test selective prediction MSE across datasets and backbones (lower is better). For each $\rho\in\{0.5,0.7,0.9\}$, acceptance thresholds are fit as empirical $\rho$-quantiles of each \textbf{raw} abstention score (\tem{} vs.\ structural) on the \textbf{validation} split and then applied to the test split without re-calibration (Protocol~B). Graph is bolded when lower than TEM at the same $\rho$. \textbf{I\%}$\uparrow$ is relative MSE reduction vs.\ TEM: $(\text{TEM}-\text{Graph})/\text{TEM}\times 100$. \textbf{S\%}$\uparrow$ is relative MSE reduction vs.\ full-split MSE: $(\text{Full}-\text{Graph}_{\rho})/\text{Full}\times 100$. Positive is better for both. \textit{N/A} indicates that the validation-derived threshold yields \textbf{zero accepted test samples} (coverage $=0$), so selective MSE and the derived I/S metrics are undefined.}
  \label{tab:selective_mse_vali_test}
  \resizebox{\textwidth}{!}{%
  % Slightly wider gaps between $\rho$ blocks than within each block (readability vs.\ width).
  \begin{tabular}{@{}ll@{\hspace{3pt}}rrrr@{\hspace{9pt}}rrrr@{\hspace{9pt}}rrrr@{\hspace{11pt}}r@{}}
    \toprule
    Dataset & Backbone & \multicolumn{4}{c}{$\rho=0.5$} & \multicolumn{4}{c}{$\rho=0.7$} & \multicolumn{4}{c}{$\rho=0.9$} & \multirow{2}{*}{\textbf{Full}} \\
    \cmidrule(lr){3-6} \cmidrule(lr){7-10} \cmidrule(lr){11-14}
     &  & TEM & Graph & \textbf{I\%}$\uparrow$ & \textbf{S\%}$\uparrow$ & TEM & Graph & \textbf{I\%}$\uparrow$ & \textbf{S\%}$\uparrow$ & TEM & Graph & \textbf{I\%}$\uparrow$ & \textbf{S\%}$\uparrow$ \\
    \midrule
    ETTh1 & Autoformer & 0.4562 & \textbf{0.3983} & +12.7 & +9.2 & 0.4418 & \textbf{0.4186} & +5.3 & +4.6 & 0.4392 & \textbf{0.4145} & +5.6 & +5.5 & 0.4388 \\
     & FEDformer & 0.4019 & \textbf{0.3100} & +22.9 & +23.4 & 0.4071 & \textbf{0.3849} & +5.5 & +4.9 & 0.4044 & \textbf{0.3989} & +1.4 & +1.5 & 0.4049 \\
     & PatchTST & 0.3810 & \textbf{0.3394} & +10.9 & +11.6 & 0.3826 & \textbf{0.3649} & +4.6 & +4.9 & 0.3836 & \textbf{0.3827} & +0.2 & +0.3 & 0.3838 \\
     & TimesNet & 0.3713 & \textbf{0.2729} & +26.5 & +29.6 & 0.3715 & \textbf{0.3146} & +15.3 & +18.9 & 0.3870 & \textbf{0.3332} & +13.9 & +14.1 & 0.3879 \\
    \addlinespace[0.2em]
    ETTh2 & Autoformer & 0.2549 & \textbf{0.0955} & +62.5 & +68.3 & 0.2598 & \textbf{0.1210} & +53.4 & +59.8 & 0.2715 & \textbf{0.1591} & +41.4 & +47.2 & 0.3011 \\
     & FEDformer & \textit{N/A} & 0.2701 & \textit{N/A} & +14.5 & 0.3384 & \textbf{0.2682} & +20.7 & +15.1 & 0.3103 & \textbf{0.2917} & +6.0 & +7.7 & 0.3159 \\
     & PatchTST & \textit{N/A} & 0.1566 & \textit{N/A} & +39.0 & \textit{N/A} & 0.2269 & \textit{N/A} & +11.7 & 0.2587 & \textbf{0.2120} & +18.1 & +17.5 & 0.2569 \\
     & TimesNet & 0.2360 & 0.2629 & -11.4 & +0.0 & 0.2519 & 0.2629 & -4.4 & +0.0 & 0.2537 & 0.2629 & -3.6 & +0.0 & 0.2629 \\
    \addlinespace[0.2em]
    ETTm1 & Autoformer & 0.4678 & \textbf{0.4170} & +10.9 & +5.5 & 0.4477 & \textbf{0.4400} & +1.7 & +0.2 & 0.4392 & 0.4411 & -0.4 & +0.0 & 0.4411 \\
     & FEDformer & 0.4648 & \textbf{0.4053} & +12.8 & +10.1 & 0.4655 & \textbf{0.4422} & +5.0 & +1.9 & 0.4538 & \textbf{0.4509} & +0.6 & +0.0 & 0.4509 \\
     & PatchTST & 0.4115 & \textbf{0.3296} & +19.9 & +8.5 & 0.3924 & \textbf{0.3528} & +10.1 & +2.0 & 0.3733 & \textbf{0.3600} & +3.6 & +0.0 & 0.3601 \\
     & TimesNet & 0.2835 & 0.3187 & -12.4 & -1.9 & 0.2859 & 0.3177 & -11.1 & -1.6 & 0.2967 & 0.3137 & -5.7 & -0.3 & 0.3127 \\
    \addlinespace[0.2em]
    ETTm2 & Autoformer & 0.1732 & \textit{N/A} & \textit{N/A} & \textit{N/A} & 0.1746 & \textbf{0.1744} & +0.1 & -1.5 & 0.1719 & \textbf{0.1662} & +3.3 & +3.3 & 0.1719 \\
     & FEDformer & 0.2440 & \textit{N/A} & \textit{N/A} & \textit{N/A} & 0.1886 & \textbf{0.1883} & +0.2 & -8.0 & 0.1325 & 0.1596 & -20.5 & +8.5 & 0.1744 \\
     & PatchTST & 0.1519 & \textit{N/A} & \textit{N/A} & \textit{N/A} & 0.1478 & 0.1653 & -11.8 & -4.2 & 0.1483 & 0.1502 & -1.3 & +5.3 & 0.1586 \\
     & TimesNet & 0.1270 & \textbf{0.1049} & +17.4 & +26.4 & 0.1328 & 0.1725 & -29.9 & -21.0 & 0.1367 & 0.1374 & -0.5 & +3.6 & 0.1426 \\
    \addlinespace[0.2em]
    Electricity & Autoformer & 0.2924 & \textbf{0.2567} & +12.2 & +12.1 & 0.2933 & \textbf{0.2565} & +12.5 & +12.2 & 0.2926 & \textbf{0.2618} & +10.5 & +10.3 & 0.2920 \\
     & FEDformer & 0.1888 & \textbf{0.1593} & +15.6 & +16.2 & 0.1894 & \textbf{0.1589} & +16.1 & +16.4 & 0.1897 & \textbf{0.1628} & +14.2 & +14.4 & 0.1901 \\
     & PatchTST & 0.1983 & \textbf{0.1508} & +24.0 & +25.7 & 0.2009 & \textbf{0.1632} & +18.8 & +19.6 & 0.2023 & \textbf{0.1671} & +17.4 & +17.7 & 0.2030 \\
     & TimesNet & 0.1769 & \textbf{0.1306} & +26.2 & +20.6 & 0.1708 & \textbf{0.1316} & +23.0 & +20.0 & 0.1661 & \textbf{0.1306} & +21.4 & +20.6 & 0.1645 \\
    \addlinespace[0.2em]
    Traffic & Autoformer & 1.0131 & \textbf{0.9573} & +5.5 & +2.2 & 0.9958 & \textbf{0.9478} & +4.8 & +3.2 & 0.9874 & \textbf{0.9538} & +3.4 & +2.6 & 0.9791 \\
     & FEDformer & 0.6946 & 0.7084 & -2.0 & -0.3 & 0.6979 & 0.7012 & -0.5 & +0.7 & 0.7046 & \textbf{0.6972} & +1.0 & +1.3 & 0.7061 \\
     & PatchTST & 0.6270 & 0.6436 & -2.6 & -9.0 & 0.6130 & \textbf{0.6122} & +0.1 & -3.7 & 0.5983 & \textbf{0.5922} & +1.0 & -0.3 & 0.5903 \\
     & TimesNet & 0.5900 & \textbf{0.5286} & +10.4 & +9.5 & 0.5858 & \textbf{0.5434} & +7.2 & +6.9 & 0.5867 & \textbf{0.5683} & +3.1 & +2.7 & 0.5839 \\
    \addlinespace[0.2em]
    Weather & Autoformer & 0.2648 & \textbf{0.2004} & +24.3 & +16.0 & 0.2514 & \textbf{0.2178} & +13.4 & +8.7 & 0.2400 & \textbf{0.2322} & +3.3 & +2.7 & 0.2386 \\
     & FEDformer & 0.3134 & \textbf{0.2300} & +26.6 & +24.3 & 0.3084 & \textbf{0.2523} & +18.2 & +17.0 & 0.3051 & \textbf{0.2933} & +3.9 & +3.5 & 0.3040 \\
     & PatchTST & 0.1792 & \textbf{0.1443} & +19.5 & +14.5 & 0.1722 & \textbf{0.1490} & +13.5 & +11.7 & 0.1701 & \textbf{0.1614} & +5.1 & +4.3 & 0.1687 \\
     & TimesNet & 0.1232 & \textbf{0.1215} & +1.4 & +11.8 & 0.1311 & \textbf{0.1251} & +4.6 & +9.2 & 0.1333 & \textbf{0.1305} & +2.1 & +5.2 & 0.1377 \\
    \bottomrule
  \end{tabular}
  }
\end{table}

\section{Multiseed robustness (Protocol B)}
\label{sec:appendix_multiseed}

\noindent Main Tables~\ref{tab:selective_mse_main}--\ref{tab:ablation2_main} use seed~2024 only. The aligned multiseed summary (Table~\ref{tab:multiseed_protocol_b}) is reported in Section~\ref{sec:threats_validity}; this appendix records extended PatchTST cells, statistical tests, and excluded partial Protocol~B settings.

\vspace{0.35em}
\noindent\textbf{PatchTST (extended complete cells).}
Adding ETTh1 $\rho{=}0.5$ and Electricity ($\rho{=}0.5,0.7,0.9$) yields 7 cells / 21 seed runs, all graph wins (21/21).
Binomial sign test ($n{=}21$): $p{\approx}1.9{\times}10^{-6}$; Wilcoxon ($n{=}21$): $p{\approx}5.1{\times}10^{-7}$ (subset test; not the full main-table grid).
\noindent\textbf{Excluded (incomplete Protocol~B cells).}
Autoformer ETTh1 $\rho{=}0.5$; Autoformer ETTh2 $\rho{=}0.5,0.7$; PatchTST ETTh2 $\rho{=}0.5,0.7$ (not all three seeds yield comparable Graph vs.\ \tem{} selective MSE under validation-calibrated thresholds).
Per-seed logs are included with the experiment code release.

\section{Backbone--dataset selective curves}
\label{sec:appendix_backbone_dataset_curves}

We provide supplementary risk--coverage visualizations for selected backbone--dataset pairs to complement the main tables. Each plot reports test selective MSE (y-axis) against target coverage (x-axis). For each target coverage level, thresholds follow the same test-split empirical quantile logic as Section~\ref{sec:experiments}: for each raw abstention score we take the empirical quantile on the test split corresponding to that coverage level.

\noindent\textbf{How to read the curves.}
Lower is better. The dashed line (\texttt{orig\_MSE}) is the full-coverage test MSE without abstention. Solid blue (\texttt{ebm\_only}) is instance-level $\Etem$; orange dashed (\texttt{graph\_only}) is $E_{\struct}$. A curve below \texttt{orig\_MSE} means abstention helps; \texttt{graph\_only} below \texttt{ebm\_only} means the structural gate ranks high-error samples better under test-split quantile gating (Section~\ref{sec:experiments}).

\noindent\textbf{Autoformer.}
\begin{figure}[H]
  \centering
  \begin{minipage}[t]{0.32\textwidth}
    \centering
    \includegraphics[width=\linewidth]{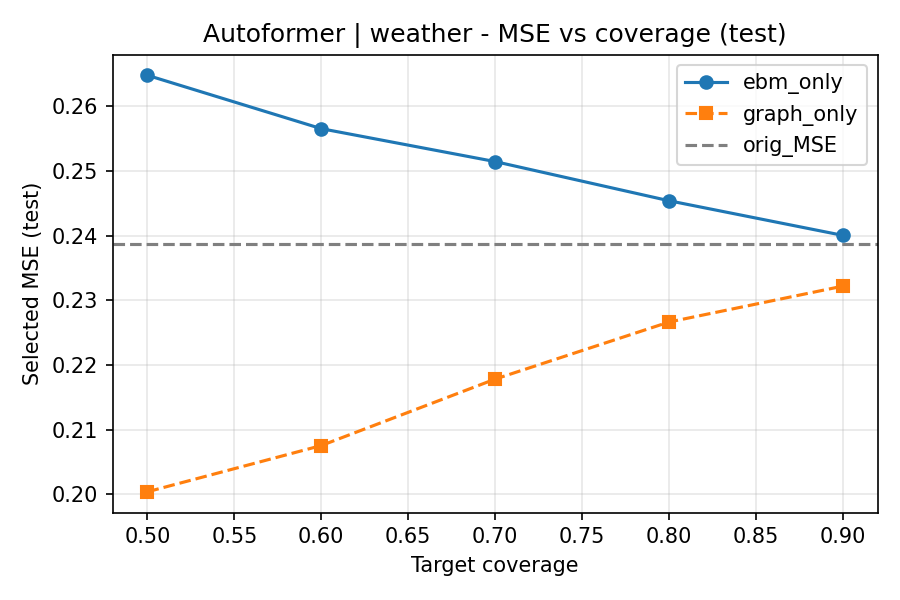}
    \vspace{-0.35em}
    {\small (a) Weather}
  \end{minipage}\hfill
  \begin{minipage}[t]{0.32\textwidth}
    \centering
    \includegraphics[width=\linewidth]{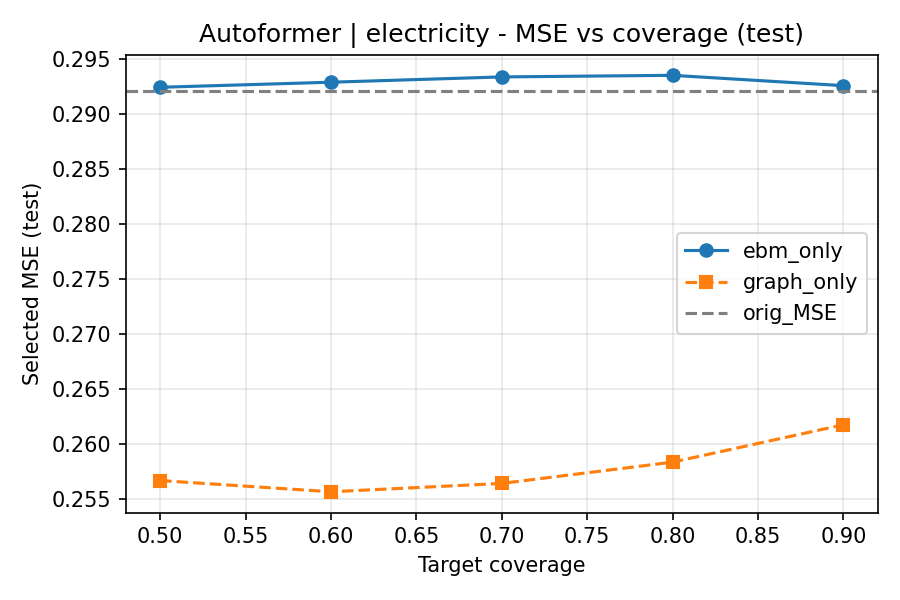}
    \vspace{-0.35em}
    {\small (b) Electricity}
  \end{minipage}\hfill
  \begin{minipage}[t]{0.32\textwidth}
    \centering
    \includegraphics[width=\linewidth]{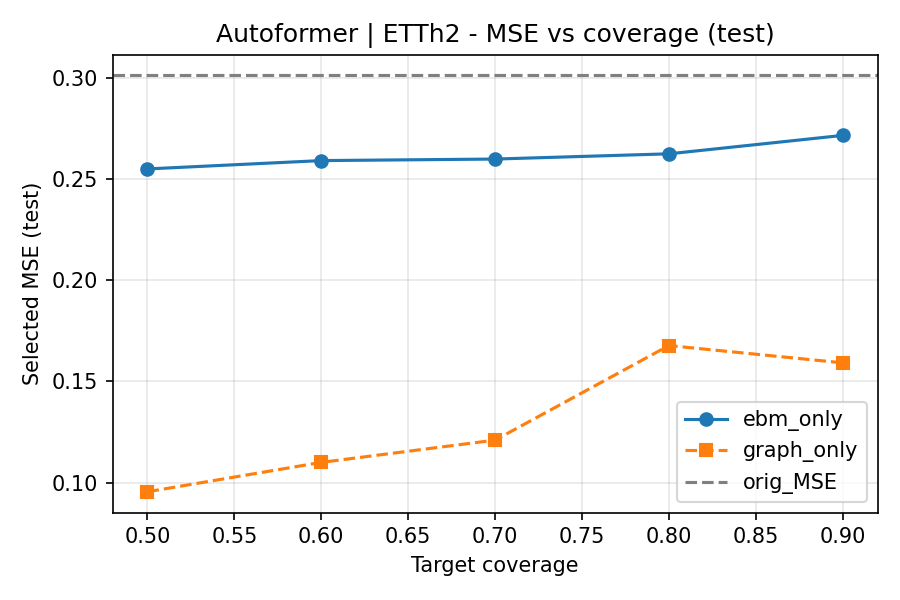}
    \vspace{-0.35em}
    {\small (c) ETTh2}
  \end{minipage}
  \caption{Autoformer: test selective MSE vs.\ target coverage on three datasets, using the same test-split quantile protocol as Section~\ref{sec:experiments}. We compare \tem{}-only (\texttt{ebm\_only}), graph-only (\texttt{graph\_only}), and the full-coverage baseline (\texttt{orig\_MSE}). Across these datasets, the structural gate typically yields substantially lower selective error at matched nominal coverage.}
  \label{fig:appendix_backbone_dataset_curves_autoformer}
\end{figure}

\noindent\textbf{PatchTST.}
\begin{figure}[H]
  \centering
  \begin{minipage}[t]{0.32\textwidth}
    \centering
    \includegraphics[width=\linewidth]{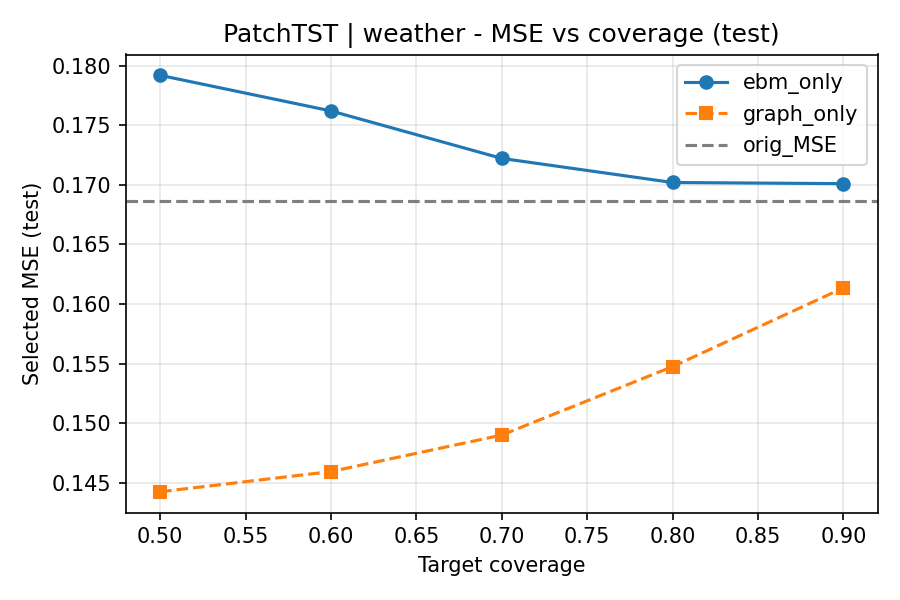}
    \vspace{-0.35em}
    {\small (a) Weather}
  \end{minipage}\hfill
  \begin{minipage}[t]{0.32\textwidth}
    \centering
    \includegraphics[width=\linewidth]{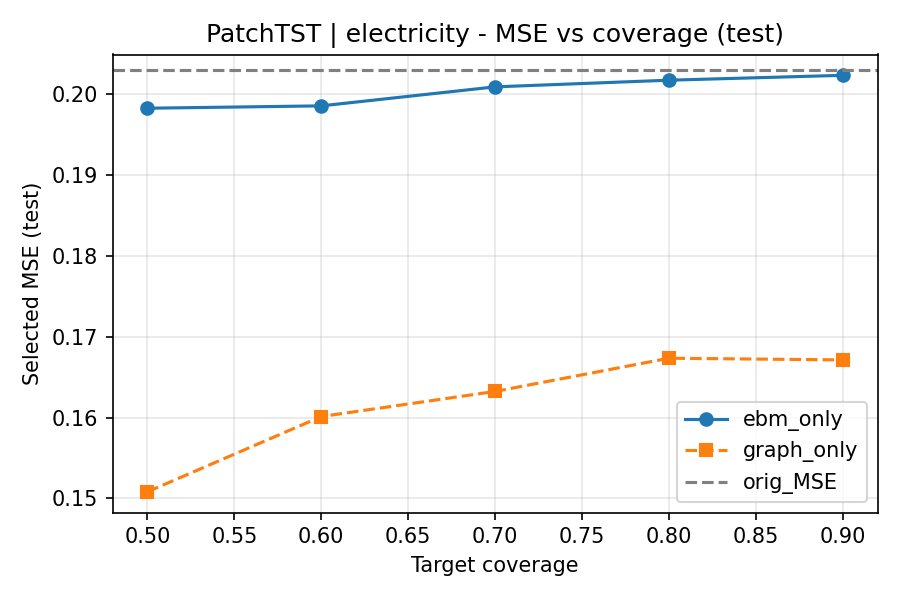}
    \vspace{-0.35em}
    {\small (b) Electricity}
  \end{minipage}\hfill
  \begin{minipage}[t]{0.32\textwidth}
    \centering
    \includegraphics[width=\linewidth]{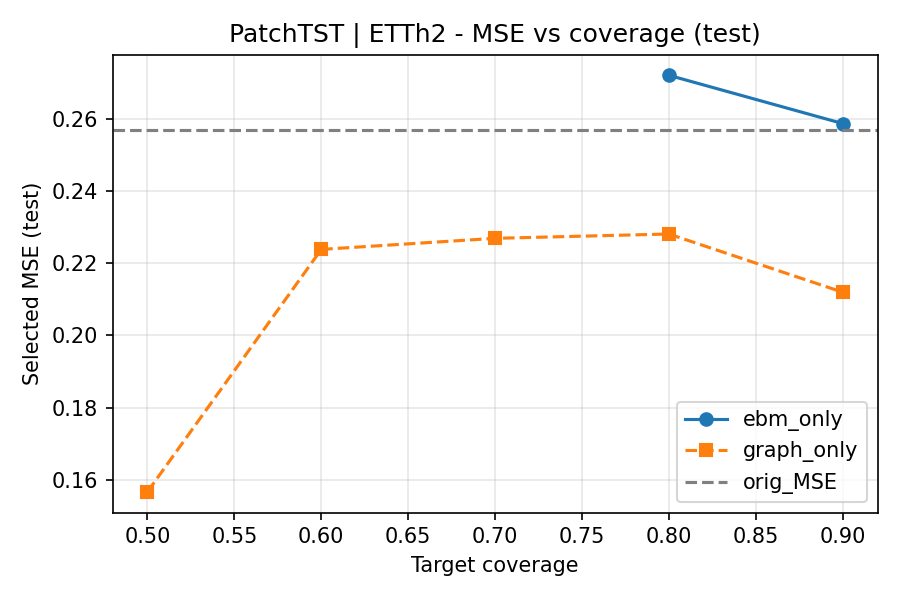}
    \vspace{-0.35em}
    {\small (c) ETTh2}
  \end{minipage}
  \caption{PatchTST: test selective MSE vs.\ target coverage on three datasets, using the same test-split quantile protocol as Section~\ref{sec:experiments}. We compare \tem{}-only (\texttt{ebm\_only}), graph-only (\texttt{graph\_only}), and the full-coverage baseline (\texttt{orig\_MSE}).}
  \label{fig:appendix_backbone_dataset_curves_patchtst}
\end{figure}

\section{Additional baseline comparisons}
\label{sec:appendix_additional_baselines}

\noindent We provide a supplementary comparison to two simple instance-level baselines (MC dropout and an error predictor) on two representative backbones across three datasets. The \tem{} comparison is already reported in the main results.

% Representative comparison to stronger instance-level baselines (main text).
% Notes:
% - Derived from baseline_compare_grid_merged/ebm_graph_fusion_test_metrics.csv
% - "EBM" corresponds to the TEM instance-energy baseline used in our pipeline.
% - Fusion rows are omitted (we do not present fusion as a contribution).

\begin{table}[H]
  \centering
  \footnotesize
  \setlength{\tabcolsep}{3.5pt}
  \caption{%
    \textbf{(Supplementary comparison.)}
    Test selective MSE on retained samples at nominal coverages on the test split (thresholds from test-split quantiles of raw scores; aligned with Table~\ref{tab:selective_mse_main}; lower is better).
    We compare our \textbf{graph-only} structural gate to \textbf{MC dropout} and an \textbf{error predictor} (instance-level baselines) on two representative backbones across three datasets (the \tem{} comparison is reported in the main table).
  }
  \label{tab:instance_baselines_main}
\vspace{0.2em}
  \begin{minipage}[t]{0.49\linewidth}
    \centering
    \textbf{PatchTST}\par\smallskip
    \begin{tabular}{@{}llrrr@{}}
      \toprule
      \textbf{Dataset} & \textbf{Method} & \textbf{50\%} & \textbf{70\%} & \textbf{90\%} \\
      \midrule
      \multirow{3}{*}{ETTh2}
        & Graph (ours) & \textbf{0.195} & \textbf{0.202} & \textbf{0.236} \\
        & Error predictor & 0.226 & 0.245 & 0.251 \\
        & MC dropout & 0.199 & 0.214 & 0.239 \\
      \addlinespace[0.15em]
      \multirow{3}{*}{Electricity}
        & Graph (ours) & \textbf{0.172} & \textbf{0.180} & \textbf{0.198} \\
        & Error predictor & 0.183 & 0.185 & 0.208 \\
        & MC dropout & 0.174 & 0.187 & 0.202 \\
      \addlinespace[0.15em]
      \multirow{3}{*}{Weather}
        & Graph (ours) & \textbf{0.134} & \textbf{0.138} & \textbf{0.155} \\
        & Error predictor & 0.140 & 0.149 & 0.156 \\
        & MC dropout & 0.158 & 0.160 & 0.160 \\
      \bottomrule
    \end{tabular}
  \end{minipage}\hfill
  \begin{minipage}[t]{0.49\linewidth}
    \centering
    \textbf{Autoformer}\par\smallskip
    \begin{tabular}{@{}llrrr@{}}
      \toprule
      \textbf{Dataset} & \textbf{Method} & \textbf{50\%} & \textbf{70\%} & \textbf{90\%} \\
      \midrule
      \multirow{3}{*}{ETTh2}
        & Graph (ours) & \textbf{0.245} & \textbf{0.269} & \textbf{0.291} \\
        & Error predictor & 0.297 & 0.290 & 0.304 \\
        & MC dropout & 0.293 & 0.296 & 0.296 \\
      \addlinespace[0.15em]
      \multirow{3}{*}{Electricity}
        & Graph (ours) & \textbf{0.263} & \textbf{0.276} & \textbf{0.286} \\
        & Error predictor & 0.281 & 0.284 & 0.287 \\
        & MC dropout & 0.279 & 0.284 & 0.288 \\
      \addlinespace[0.15em]
      \multirow{3}{*}{Weather}
        & Graph (ours) & \textbf{0.190} & \textbf{0.201} & \textbf{0.224} \\
        & Error predictor & 0.214 & 0.220 & 0.227 \\
        & MC dropout & 0.196 & 0.205 & 0.226 \\
      \bottomrule
    \end{tabular}
  \end{minipage}
\end{table}

\section{Theoretical analysis}
\label{sec:appendix_theory}

\noindent This appendix connects the implemented structural score to selective forecasting: learned relations $\rightarrow$ structural energy $\rightarrow$ forecast-error certificates $\rightarrow$ score ranking $\rightarrow$ selective risk. All statements analyze the energy in Eq.~\eqref{eq:estruct} without modifying the algorithm.

\subsection{Symmetrized structural energy}
\label{sec:appendix_theory_sym}

For $Z\in\R^{H\times D}$ (horizon $\times$ variables), write $E_{\struct}(Z,A)$ as in Eq.~\eqref{eq:estruct} with $\mathrm{dist}_{i,j}(Z)=\|Z_{:,i}-Z_{:,j}\|_2^2$. Define symmetrized weights $W=\tfrac12(A+A^\top)$.

\begin{lemma}[Symmetrization invariance]
\label{lem:symmetrization}
For any nonnegative $A$ and trajectory $Z$, $E_{\struct}(Z,A)=E_{\struct}(Z,W)$.
\end{lemma}
\begin{proof}
Since $\mathrm{dist}_{ij}=\mathrm{dist}_{ji}$,
$\sum_{i,j}A_{ij}\mathrm{dist}_{ij}=\tfrac12\sum_{i,j}(A_{ij}+A_{ji})\mathrm{dist}_{ij}=\sum_{i,j}W_{ij}\mathrm{dist}_{ij}$.
Scaling by $1/D^2$ is unchanged.
\end{proof}

For analysis, let $D_W=\mathrm{diag}(W\mathbf{1})$ and $L_W=D_W-W\succeq 0$. Then $\|Z\|_W\coloneqq\sqrt{E_{\struct}(Z,A)}$ is a seminorm on $\R^{H\times D}$. If $W$ is connected on variables, $\|Z\|_W=0$ exactly when all columns of $Z$ coincide (trajectories constant across variables); restricted to $\{Z\in\R^{H\times D}: Z\mathbf{1}_D=\mathbf{0}\}$, $\|\cdot\|_W$ is a norm. Theorem~\ref{thm:error_certificate} uses the seminorm triangle inequality on $\R^{H\times D}$; test-time scores are unchanged.

\subsection{Structural energy as an error certificate}
\label{sec:appendix_theory_energy}

\paragraph{Assumption 1 (graph-smooth ground truth).}
For ground-truth future $Y$, $E_{\struct}(Y,A)\le \epsilon$ for the (learned) adjacency $A$ used at scoring time.

\begin{theorem}[Structural-energy error certificate]
\label{thm:error_certificate}
Assume $\lambda_{\max}(L_W)>0$. For any forecast $\hat{Y}$,
\[
\|\hat{Y}-Y\|_F^2
\;\ge\;
\frac{D^2}{2\lambda_{\max}(L_W)}
\left[
\sqrt{E_{\struct}(\hat{Y},A)}-\sqrt{\epsilon}
\right]_+^{2},
\qquad [x]_+=\max(x,0).
\]
\end{theorem}
\begin{proof}
By the triangle inequality for $\|\cdot\|_W$,
$\|\hat{Y}-Y\|_W\ge\big[\|\hat{Y}\|_W-\|Y\|_W\big]_+$.
Squaring and using $E_{\struct}(Y,A)\le\epsilon$ gives
$E_{\struct}(\hat{Y}-Y,A)\ge\big[\sqrt{E_{\struct}(\hat{Y},A)}-\sqrt{\epsilon}\big]_+^2$.
For residual $R=\hat{Y}-Y$, pairwise expansion and $L_W\succeq 0$ yield
$E_{\struct}(R,A)=\tfrac{1}{D^2}\sum_{ij}W_{ij}\|R_{:,i}-R_{:,j}\|_2^2
\le \tfrac{2\lambda_{\max}(L_W)}{D^2}\|R\|_F^2$.
Combine the two displays.
\end{proof}

\paragraph{Corollary 1 (per-sample selective error).}
With $e_b=\|\hat{Y}_b-Y_b\|_F^2/(HD)$ (Eq.~\eqref{eq:mse}), Theorem~\ref{thm:error_certificate} is equivalent to Eq.~\eqref{eq:cert_eb_main} in Section~\ref{sec:method} (divide the Frobenius bound by $HD$). If $E_{\struct}(Y,A)=0$, then $e_b\ge \tfrac{D}{2H\lambda_{\max}(L_W)}E_{\struct}(\hat{Y},A)$: large structural energy is a \emph{one-sided certificate} on the same error used in selective risk, not an exact MSE predictor.

\subsection{Selective ranking and excess risk}
\label{sec:appendix_theory_ranking}

On a test batch of size $B$, let per-sample errors $e_1,\ldots,e_B$ (Eq.~\eqref{eq:mse}) and target coverage $c\in(0,1]$, $k=\lfloor cB\rfloor$.
For abstention score $S=(S_1,\ldots,S_B)$, let $\mathcal{S}_S(k)$ be the $k$ samples with smallest scores and
$R_S(k)=\frac{1}{k}\sum_{b\in\mathcal{S}_S(k)} e_b$.
Let $\mathcal{S}^*(k)\in\arg\min_{|\mathcal{T}|=k}\sum_{b\in\mathcal{T}} e_b$ and $R^*(k)=\frac{1}{k}\sum_{b\in\mathcal{S}^*(k)} e_b$.

\begin{theorem}[Oracle ranking optimality]
\label{thm:oracle_ranking}
For every score $S$ and every $k$, $R_S(k)\ge R^*(k)$. If $e_i<e_j\Rightarrow S_i<S_j$ for all $i,j$, then $R_S(k)=R^*(k)$ for all $k$.
\end{theorem}
\begin{proof}
Among $k$-subsets, $\mathcal{S}^*(k)$ minimizes $\sum_{b\in\mathcal{T}}e_b$ by definition.
If $S$ is strictly order-consistent with $e$, the lowest-$k$ scores coincide with the lowest-$k$ errors, so $\mathcal{S}_S(k)=\mathcal{S}^*(k)$.
\end{proof}

\begin{theorem}[Excess selective risk under disagreement]
\label{thm:excess_risk}
Assume $0\le e_b\le E_{\max}$ for all $b$. Let $m_k=|\mathcal{S}_S(k)\setminus\mathcal{S}^*(k)|$ and $\delta_k=m_k/k$. Then $0\le R_S(k)-R^*(k)\le E_{\max}\delta_k$.
\end{theorem}
\begin{proof}
Let $F=\mathcal{S}_S(k)\setminus\mathcal{S}^*(k)$ and $M=\mathcal{S}^*(k)\setminus\mathcal{S}_S(k)$; $|F|=|M|=m_k$.
Canceling common samples,
$R_S(k)-R^*(k)=\frac{1}{k}(\sum_{i\in F}e_i-\sum_{j\in M}e_j)\le \frac{m_k E_{\max}}{k}$.
\end{proof}

Theorem~\ref{thm:oracle_ranking} matches coverage-constrained selective forecasting (Section~\ref{sec:experiments}): under fixed $\rho$, we sort raw scores and retain the lowest fraction.
Theorem~\ref{thm:excess_risk} bounds suboptimality by the fraction of mis-ranked acceptances at size $k$ (finite-batch, empirical statement).

\subsection{Alignment objective and implementation}
\label{sec:appendix_theory_alignment}
\label{sec:appendix_alignment_impl}

Training uses $\alignloss=-r(s,t)$ with $s_b=E_{\struct}(\hat{Y}_b,A)$ and default $t_b=\log(1+e_b)$.
The Pearson objective does \emph{not} directly optimize permutation agreement; it is a smooth surrogate encouraging positive score--risk association.

\begin{lemma}[Monotone target]
\label{lem:monotone_target}
For $t(e)=\log(1+e)$, $e_i<e_j\iff t(e_i)<t(e_j)$; oracle error ranking is unchanged by the log target.
\end{lemma}

\begin{proposition}[Pairwise ranking stability under approximate alignment]
\label{prop:pairwise_stability}
Suppose at evaluation time $t_b=\alpha s_b+\beta+\xi_b$ with $\alpha>0$ and $|\xi_b|\le\eta$.
If $s_j-s_i>2\eta/\alpha$, then $t_j>t_i$ (hence $e_j>e_i$ when $t=\log(1+e)$ and the map is invertible on the observed range).
\end{proposition}
\begin{proof}
$t_j-t_i=\alpha(s_j-s_i)+(\xi_j-\xi_i)\ge \alpha(s_j-s_i)-2\eta>0$.
\end{proof}

Thus correlation alignment is \emph{consistent with} the ranking objective under an approximately positive affine score--risk relation, but does not guarantee perfect ranking on finite batches.
Differentiable Pearson correlation: with $\tilde{s}_b=s_b-\bar{s}$, $\tilde{t}_b=t_b-\bar{t}$,
\begin{equation}
r(s,t)=\frac{\sum_{b=1}^{B}\tilde{s}_b\,\tilde{t}_b}{\sqrt{\sum_{b=1}^{B}\tilde{s}_b^{2}+\varepsilon}\,\sqrt{\sum_{b=1}^{B}\tilde{t}_b^{2}+\varepsilon}},
\label{eq:pearson_batch}
\end{equation}
with fixed $\varepsilon>0$; gradients flow through $s_b$ into the graph head and forecaster.

\subsection{Structure-blind scoring limits}
\label{sec:appendix_theory_failure}

A \emph{structure-blind} score $S_{\mathrm{marg}}(Z)=\sum_{d=1}^{D}\phi_d(Z_{:,d})$ depends only on per-channel summaries, not on cross-channel gaps.

\begin{proposition}[Marginal scores need not detect relational deviation]
\label{prop:structure_blind}
There exist $Z^{(1)},Z^{(2)}$ and nonnegative $A$ such that $S_{\mathrm{marg}}(Z^{(1)})=S_{\mathrm{marg}}(Z^{(2)})$ while $E_{\struct}(Z^{(1)},A)\neq E_{\struct}(Z^{(2)},A)$.
\end{proposition}
\begin{proof}
Take $D=2$ with $A_{12}=A_{21}>0$. Let $x\in\R^H$ and a permutation $P$ with $\phi_1(x)=\phi_2(x)=\phi_1(Px)=\phi_2(Px)$.
Set $Z^{(1)}=[x\; x]$ and $Z^{(2)}=[x\; Px]$ (columns).
Then $S_{\mathrm{marg}}$ matches, but $E_{\struct}(Z^{(2)},A)\propto\|x-Px\|_2^2$ can exceed $E_{\struct}(Z^{(1)},A)=0$.
\end{proof}

This concerns separable structure-blind rules; \tem{} is our learned instance-level comparator in experiments---we do not assume it equals $S_{\mathrm{marg}}$, but the example motivates a distinct relational abstention axis.

\subsection{Scope and assumptions}
\label{sec:appendix_theory_scope}

\noindent\textbf{A1.} Adjacency entries are nonnegative; analysis uses $W=(A+A^\top)/2$ (Lemma~\ref{lem:symmetrization}).
\textbf{A2.} Theorem~\ref{thm:error_certificate} assumes $\lambda_{\max}(L_W)>0$ and Assumption~1 for the same $A$ deployed at test time (learned $A$ may mis-specify coupling, regime~(iii) in Section~\ref{sec:when_helps}).
\textbf{A3.} Theorems~\ref{thm:oracle_ranking}--\ref{thm:excess_risk} are finite-batch statements for empirical selective risk at nominal $k$.
\textbf{A4.} Proposition~\ref{prop:pairwise_stability} is conditional on post-training approximate alignment, not implied in full generality by minimizing Pearson correlation on training batches.
These results do \emph{not} imply universal superiority of $E_{\struct}$; when ground-truth smoothness is weak or error is predominantly marginal, certificates and relational scores are expected to be less informative, consistent with Section~\ref{sec:when_helps}.

\section{Compute resources}
\label{sec:appendix_compute}

All experiments were run with PyTorch and CUDA on a single NVIDIA L4 GPU (24GB per device). Software versions are pinned in \texttt{environment.yml}; default training schedules follow the project \texttt{README}.

\end{document}